\documentclass{article}

\usepackage[nonatbib,preprint]{neurips_2020}
\usepackage[numbers,sort&compress]{natbib}

\usepackage[utf8]{inputenc} % allow utf-8 input
\usepackage[T1]{fontenc}    % use 8-bit T1 fonts
\usepackage[colorlinks, allcolors=black]{hyperref}       % hyperlinks

\usepackage{alex}
\usepackage{mathtools}

\newcommand{\tbf}[1]{\textbf{#1}}
\DeclarePairedDelimiter{\abs}{\lvert}{\rvert}

\DeclarePairedDelimiter{\norm}{\lVert}{\rVert}

\newcommand{\indep}{\protect\mathpalette{\protect\independenT}{\perp}}
\def\independenT#1#2{\mathrel{\rlap{$#1#2$}\mkern2mu{#1#2}}}

\newcommand{\trm}[1]{\mathrm{#1}}

\providecommand\f[2]{\ensuremath \frac{#1}{#2}}

\providecommand\f[2]{\ensuremath \frac{#1}{#2}}

\newcommand \heading[1] {\noindent \tbf{#1}.}

\usepackage{bm,graphicx,placeins,float,relsize,wrapfig,listings,mathabx,hyperref,grffile,colortbl,dsfont,bbm, float,nicefrac,amsfonts,microtype,url,booktabs,subcaption}
\usepackage[shortlabels]{enumitem}
\usepackage[percent]{overpic}
\usepackage[usenames,dvipsnames,svgnames,table]{xcolor}
\usepackage[font={small}]{caption}

\usepackage{multibib} % allows 2 bibliographies
\newcites{si}{Additional References for the Appendix}

\newcommand{\nocontentsline}[3]{}
\let\oldaddcontentsline\addcontentsline
\newcommand{\tocless}[2]{%
  \let\addcontentsline\nocontentsline
  #1{#2}
  \let\addcontentsline\oldaddcontentsline}

\renewcommand{\cite}{\citep} % overwrite \cite with \citep

\newcommand{\beginsupplement}{ % use to mark beginning of supplementary section.
        \setcounter{section}{0}
        \renewcommand{\thesection}{S\arabic{section}} %
         \renewcommand{\thesubsection}{\thesection.\arabic{subsection}}
        \setcounter{table}{0}
        \renewcommand{\thetable}{S\arabic{table}} %
        \setcounter{figure}{0}
        \renewcommand{\thefigure}{S\arabic{figure}} %
     }

\definecolor{LightGray}{gray}{0.9}

\newcolumntype{g}{>{\columncolor{LightGray}}c} % use g instead of c in tables to shade those columns

\newcommand\papertitle{Fast, Accurate, and Simple Models for Tabular Data via Augmented Distillation}

\newcommand\transformer{FAST-DAD-Net}
\newcommand\method{FAST-DAD}

\title{\papertitle}

\author{
Rasool Fakoor\thanks{Equal contribution.} \\
Amazon Web Services \\
\texttt{fakoor@amazon.com} \\
\And 
Jonas Mueller\footnotemark[1] \\
Amazon Web Services \\
\texttt{jonasmue@amazon.com} \\[-1em]
  \AND
  Nick Erickson\\
  Amazon Web Services\\
  \texttt{neerick@amazon.com} \\
  \And
  Pratik Chaudhari \\
  University of Pennsylvania \\
  \texttt{pratikac@seas.upenn.edu}
  \And
  Alexander J. Smola \\
  Amazon Web Services \\
  \texttt{smola@amazon.com} \\
}

\begin{document}

\maketitle

\vspace*{-1em}
\begin{abstract}
  Automated machine learning (AutoML) can
  produce complex model ensembles by stacking, bagging,
  and boosting many individual models like trees,
  deep
  networks, and nearest neighbor estimators. While highly accurate,
  the resulting predictors are \emph{large, slow}, and \emph{opaque} as
  compared to their constituents.
  To improve the deployment of AutoML on tabular data, we propose FAST-DAD to distill arbitrarily-complex ensemble predictors into individual models like boosted trees, random forests, and deep networks.
  At the heart of our
  approach is a data augmentation strategy based on Gibbs sampling from a self-attention pseudolikelihood estimator.
  Across 30 datasets spanning regression and binary/multiclass
  classification tasks, FAST-DAD distillation
  produces significantly better individual models than one obtains through standard training on the original data.
 Our individual distilled models are over 10$\times$ faster \emph{and} more accurate than ensemble predictors produced by AutoML tools like H2O/AutoSklearn.
  % The models are 10,000x faster than the ones in AutoGluon with minimal loss of accuracy. In fact, they're also over 10x faster and more accurate than any other toolkit (H2O, Auto-ScikitLearn).
\end{abstract}

\section{Introduction}

\begin{wrapfigure}{r}{0.45\textwidth}
\vspace*{-7mm}
\centering
    \includegraphics[width=0.45\textwidth]{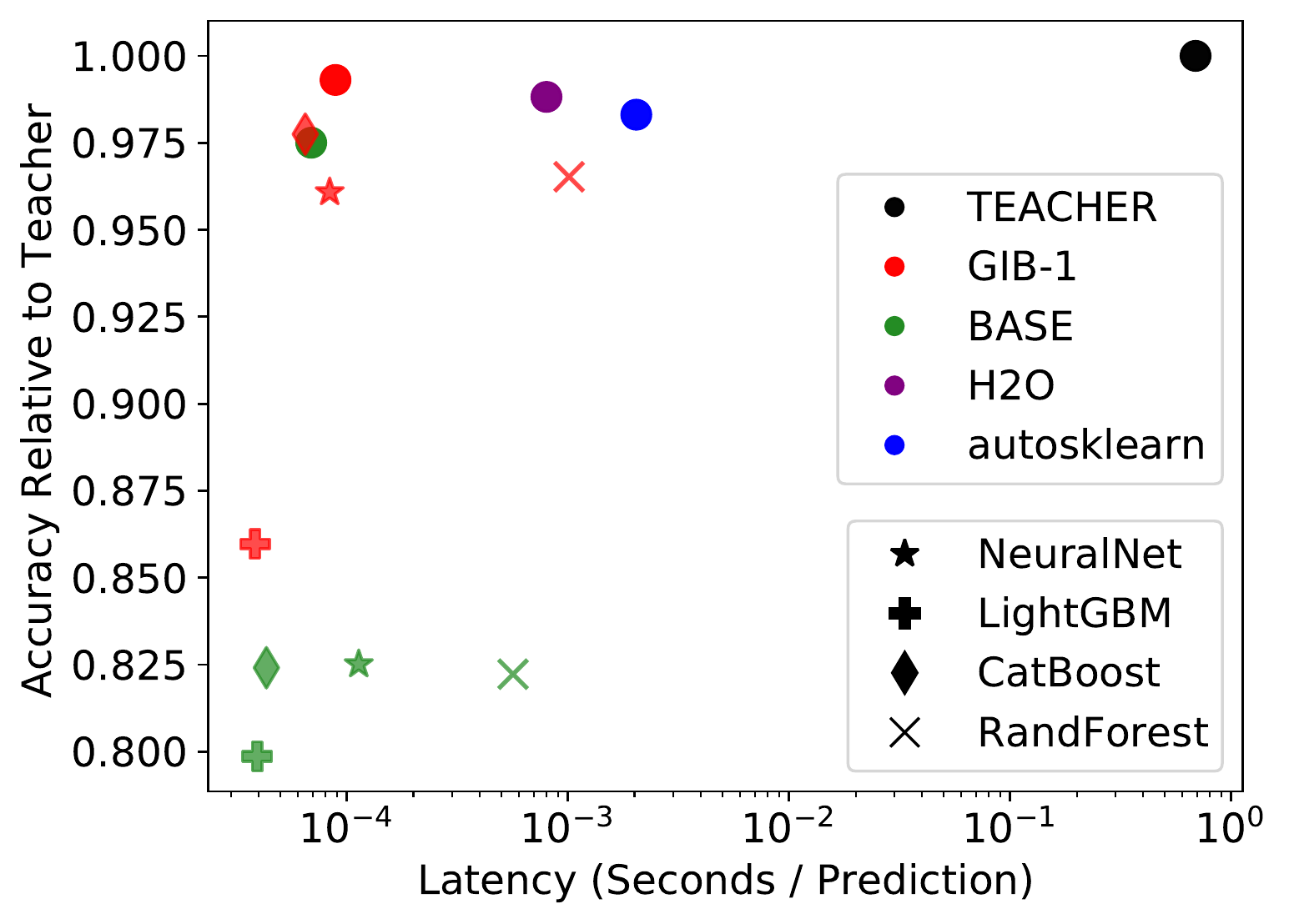}
    \caption{\tbf{Normalized test accuracy vs.\ speed 
      of individual models and AutoML ensembles, averaged over all 30 datasets}. TEACHER denotes the performance of AutoGluon; H2O and autosklearn represent the respective AutoML tools. GIB-1 indicates the results of \method{} after 1 round of Gibbs sampling. BASE denotes the student model fit on original data.  GIB-1/BASE dots represent the model \emph{Selected} (out of the 4 types) based on validation accuracy for each dataset.  
      }  
    \label{fig:improvements-hero}
     \vspace{-2mm}
\end{wrapfigure}
Modern AutoML tools provide good out-of-the-box accuracy on diverse datasets. This is often achieved through extensive model ensembling~\cite{erickson2020autogluon,feurer2019auto,cortes2017adanet}.
While the resultant predictors may generalize well, they can be
large, slow, opaque, and expensive to deploy. \cref{fig:improvements-hero} shows that the most accurate predictors can be 10,000 times slower than their constituent models.

Model distillation~\citep{bucilua2006compress,hinton2015distilling} offers a way to compress the knowledge learnt by these complex models into simpler predictors with reduced inference-time and memory-usage that are also less opaque and easier to modify and debug.
In distillation, we train a simpler model (the
\emph{student}) to output similar predictions as those of a more complex model
(the \emph{teacher}). Here we use AutoML to create the most accurate
possible teacher, typically an ensemble of many individual models via
stacking, bagging, boosting, and weighted
combinations~\citep{dietterich2000ensemble}.
Unfortunately,
distillation typically comes with a sharp drop in accuracy.
Our paper mitigates this drop via \method{}, a technique to produce \textbf{Fast}-and-accurate models via \textbf{D}istillation with \textbf{A}ugmented \textbf{D}ata. We apply \method{} to large
stack-ensemble predictors from AutoGluon \cite{erickson2020autogluon}
to produce individual models that are over 10,000$\times$ faster than AutoGluon and over 10$\times$ faster, yet still more accurate, than ensemble predictors produced by H2O-AutoML \cite{H2O} and AutoSklearn \cite{feurer2019auto}. 
% These distilled models are also more accurate than all predictors but their teacher.

% \section{Motivation}
\heading{Motivation}
A key issue in distillation is that the quality of the student is largely determined by the amount of available training data. 
While standard distillation confers smoothing benefits (where the teacher may provide higher-quality prediction targets to the student \cite{hinton2015distilling, tang2020understanding}), it incurs a student-teacher statistical approximation-error of similar magnitude as when training directly on original labeled dataset. 
% , i.e.\ of order $O(\sqrt{\mathrm{VC}/m})$. Here VC denotes the VC dimension of the problem and $m$ is the sample size. 
By increasing the amount of data available for distillation, one can improve the student's approximation of the teacher and hence the student's accuracy on test data (assuming that the teacher achieves superior generalization error than fitting the student model directly to the original data). The extra data need not be labeled; one may use the teacher to
label it. This enables the use of density estimation techniques to
learn the distribution of the training data and draw samples of
unlabeled data. In fact, we need not even learn the full joint distribution but simply learn how to \emph{draw approximate samples} from it. We show that the statistical error in these new samples can be traded off against the reduction in variance from fitting the student to a larger dataset. Our resultant student models are almost as accurate as the teacher while being far more efficient/lightweight.

% \tbf{The contributions of this paper are as follows:}
The contributions of this paper are as follows:
\begin{enumerate}[noitemsep,topsep=0pt,parsep=0pt,partopsep=0pt]
\item We present model-agnostic distillation that works across many
  types of teachers and students and different supervised learning
  problems (binary and multiclass classification, regression). This is
  in contrast to problem and architecture-specific distillation
  techniques
  \cite{bucilua2006compress,hinton2015distilling,vidal2020born,cho2019efficacy}.
\item We introduce a maximum pseudolikelihood model for
  tabular data that uses self-attention across covariates to
  simultaneously learn all of their conditional distributions.
\item We propose a corresponding Gibbs sampler that takes advantage of these 
  conditional estimates to efficiently augment the dataset used in distillation. Our \method{} approach avoids estimating the features' joint distribution, and enables
  control over sample-quality and diversity of the augmented dataset.
\item We report the first comprehensive distillation benchmark for
  tabular data which studies 5 distillation strategies with 4
  different types of student models applied over 30 datasets involving
  3 different types of prediction tasks.
\end{enumerate}

Although our techniques can be adapted to other
modalities, we focus on tabular data which has been under-explored in
distillation despite its ubiquity in practical
applications. Compared to typical data tables, vision and language datasets have far larger sample-sizes and with easily available data; data augmentation is thus not as critical for distillation as it is in the tabular setting.

%The novel contributions of this work include:

%- distillation with many student model types; many tasks.  Most existing work only considers distilling student models that are NN (to gain efficiency) or decision trees (to gain interpretability)

%- multiclass distillation with trees (especially GBM) is novel

%- filtering of augmented data (+ modification to MUNGE augmentation)

%- Distillation within an overall ML pipeline (eg model selection among potential students and original base models)

\vspace*{-0.3em}
% \subsection{Related Work}
\section{Related Work}

While distillation and model compression are popular
in deep learning, existing work focuses primarily on vision, language
and speech applications. Unlike the tabular settings we consider here, this prior work studies situations where: (a) unlabeled data is plentiful; (b) there are many more training examples than in typical data tables; (c) both teacher and student are neural networks; (d) the task is multiclass classification
\cite{BaCaruana14,hinton2015distilling,urban2016deep,cho2019efficacy,mirzadeh2019improved,
  yang2020knowledge}.
% These image datasets tend to be massive and thus augmentation is not
% as critical as it is in the tabular setting.

For tabular data, \citet{breiman96} considered distilling
models into single decision trees, but this often unacceptably harms accuracy. Recently, \citet{vidal2020born} showed how to distill tree ensembles into a single tree without sacrificing accuracy, but their approach is restricted to tree student/teacher models. Like us,
\citet{bucilua2006compress} considered distillation of large ensembles
of heterogeneous models, pioneering the use of data augmentation in
this process. Their work only considered binary classification problems with a neural network student model; multiclass
classification is handled in a one-vs-all fashion which produces
less-efficient students that maintain a model for every
class. \citet{liu2018teacher} suggest generative-adversarial networks
can be used to produce better augmented data, but only
conduct a small-scale distillation study with random forests.
% \citet{hinton2015distilling} proposed a popular distillation method designed for neural network teacher/student models applied to multiclass classification problems with large labeled datasets. % Can cut this sentence
% \fi

% \todo{One can alternatively learn a generator that directly produces samples from some implicit estimated distribution~\citep{xu2019modeling}.} \jonas{added GAN-TSC ref above which addresses this}

\vspace*{-0.2em}
\section{From Function Approximation to Distillation}
\label{s:motivation}

% This section formalizes model distillation and how to improve the student using unlabeled data sampled from some auxiliary (learned) distribution.
% using a density estimator trained on the training dataset

We first formalize distillation  
to quantify the role of the auxiliary data in this process.
Consider a dataset $(X_n, Y_n)$ where
$X_n = \cbr{x_i \in \mathcal{X} \subset \RR^d }_{i=1}^n $ are observations of some features sampled from distribution $p$, and 
$Y_n = \cbr{y_i \in \Ycal}_{i=1}^n$ are their labels 
sampled from distribution $p_{y|x}$. The teacher
$f: \Xcal \to \Ycal$ is some function learned e.g.\ via AutoML that achieves good generalization error:
$$R[f] := \Eb_{(x,y) \sim p} \sbr{\ell(f(x), y)}$$
where loss $\ell$ measures the error in individual predictions.
Our goal is to find a model $g$ from a restricted class of
functions $\Gcal \subset \Lcal^2(\Xcal)$ such that $R[g]$ is
smaller than the generalization error of another model from this class produced via empirical risk minimization.

% Since we don't have access to $p$ or $R[\cdot]$ it is difficult to distill such a student $g$ directly.

% When employing distillation, we assume $R[f] \ll R[\widehat{g}]$, where $\widehat{g} = \argmin_{g \in \Gcal} R_\trm{emp}[g]$  is the model one would traditionally obtain by minimizing the empirical risk $R_\trm{emp}[g] = \f{1}{n} \sum_{i=1}^n \ell(g(x_i), y_i)$.
%
%is trained to minimize the empirical risk $R_{\trm{emp}}[f] = \f{1}{n} \sum_{i=1}^n \ell(f(x_i), y_i)$ where loss $\ell$ measures the prediction-error. Our goal is to learn a student, denoted by $g \in \Gcal \subset \Lcal^2(\Xcal)$ with the understanding that (i) the teacher $f \in \Fcal$ trained on $X_n$ generalizes well,
% and (ii) $\Gcal \subset \Fcal$, i.e., the student has a smaller capacity than the teacher. J: we dont need this assumption, this is a DL-distillation assumption but teacher-capacity is not especially relevant in our context

\heading{Approximation}
%The key idea in distillation is to find 
Distillation seeks 
some student $g^*$ that is
``close'' to the teacher $f$. If
$\nbr{f-g^*}_{\infty} \leq \epsilon$ over $\Xcal$ and if the loss function
$\ell$ is Lipschitz continuous in $f(x)$, then $g^*$ will be nearly as accurate as the teacher ($R\sbr{g^*} \approx R\sbr{f}$). Finding such a $g^*$ may however be impossible. For instance, a Fourier
approximation of a step function will never converge uniformly but only
pointwise. This is known as the \emph{Gibbs phenomenon}
\cite[]{wilbraham1848certain}.  Fortunately, $\ell_\infty$-convergence is not required: we only
require convergence with regard to some distance function
$d(f(x), g(x))$ averaged over $p$. Here $d$ is determined by the
task-specific loss $\ell$. For instance, $\ell_2$-loss can be used for
regression and the KL divergence between class-probability estimates
from $f,g$ may be used in classification. Our goal during
distillation is thus to minimize
\begin{align}\label{eq:truedist}
    D(f,g,p) = \Eb_{x \sim p}\sbr{d(f(x), g(x))}.
\end{align}
This is traditionally handled by minimizing its empirical counterpart \citep{hinton2015distilling}:
\begin{align}
    \label{eq:truedist-emp}
    D_\trm{emp}(f,g,X_n) = \frac{1}{n} \sum_{i=1}^n d( f(x_i), g(x_i)).
\end{align}
\heading{Rates of Convergence}
Since it is only an empirical average, minimizing 
$D_\trm{emp}$ over $g \in \Gcal$ will give rise to an
approximation error that can bounded, e.g.\ by uniform convergence
bounds from statistical learning theory
\cite{vapnik1998statistical}  as 
$O(\sqrt{V/n})$. Here $V$ denotes the complexity
of the function class $\Gcal$ and $n$ is the number
of observations used for distillation. Note that we effectively pay
twice for the statistical error due to sampling $(X_n, Y_n)$ from
$p$. Once to learn $f$ and again while distilling $g^*$ from $f$ using the same samples.

There are a number of mechanisms to reduce the second error. If we had
access to more unlabeled data, say $X_m'$ with $m \gg n$ drawn from
$p$, we could reduce the statistical error due to distillation
significantly (see \cref{fig:unlabeledperf}).
% In fact, our experiments (see \cref{fig:unlabeledperf}) show precisely that.
% J: Because there is abundance of labeled data in famous benchmarks, most distillation papers in those fields do not use unlabled data even though I totally agree they  could/should.
% This is one of the reasons why distillation is quite popular with text and images where such data is plentiful.
While we usually cannot draw from $p$ for tabular data due to a lack of additional unlabeled examples (unlike say for images/text), we might be able to draw from a related
distribution $q$ which is sufficiently close. In this case we
can obtain a uniform convergence bound:
\vspace*{-0.2em}
\begin{lemma}[Surrogate Approximation]
  \label{lem:surrogate}
      Assume that the complexity of the function class $\mathcal{G}$
      is bounded under $d(f(x), \cdot)$ and
      $d(f(x), g(x)) \leq 1$ for all $x \in \Xcal$ and $g \in
      \Gcal$. Then there exists a constant $V$ such that
      with probability at least $1 - \delta$ we have
      \begin{align}
        \label{eq:surrogate}
        D(f, g^*, p) \leq D_{\mathrm{emp}}(f, g^*, X_m') + \sqrt{\rbr{V - \log \delta}/{m}} + \nbr{p-q}_1.
      \end{align}
      Here $X_m'$ are $m$ samples from $q$ and $g^* \in \Gcal$
      is chosen, e.g.\ to minimize $D_{\mathrm{emp}}(f, g, X_m')$.
\end{lemma}
\vspace*{-0.2em}
\begin{proof}
  This follows directly from H\"older's inequality when applied to
  $D(f,g,p) - D(f,g,q) = \int l(f(x), g(x)) (p(x) - q(x)) dx \leq C
  \|p-q\|_1$.
  Next we apply uniform convergence bounds to the difference between
  $D_{\mathrm{emp}}(f,g,X_m') - D(f,g,q)$. Using VC bounds
  \cite{vapnik1998statistical} proves the claim.
\end{proof}
The inequality~\cref{eq:surrogate} suggests a number of strategies when
designing algorithms for distillation.
% \begin{itemize*}
% \item Whenever we have plenty of unlabeled data drawn from the true distribution the second term in \eq{eq:surrogate} vanishes. We can make the first term  arbitrarily small by picking a large $m$ with only the approximation error left.
  % J: Im not aware of this being done much. Most labeled image/text distillation datasets are already large.  This is one of the reasons for the effectiveness of distillation for images and text.
%  \end{itemize} 
 Whenever $p$ and $q$ are similar in terms of the bias $\nbr{p-q}_1$ being small, 
  we want to draw as much data as we can from $q$ to make the uniform convergence term vanish. 
  % After all, the effects of the approximation are quite negligible there.
 However if $q$ is some sort of estimate, a nontrivial difference between $p$ and $q$ will usually exist in practice. In this case, we may trade off the variance reduction offered by extra augmented samples and the corresponding bias by drawing these samples from an intermediate distribution that lies \emph{in between} the training data and the biased $q$.
 
% In this case we can reduce it by mixing sample drawn from $q$ with those drawn from $p$. This amounts t drawing from a distribution $\frac{n}{n+m} p + \frac{m}{n+m} q$ and thus a reduction of the approximation error from $\nbr{p-q}_1$ to $\frac{m}{n+m} \nbr{p-q}_1$.
%
% Note that the above guarantees are independent of \emph{how} the samples from $q$ were obtained. In particular, they need not be drawn iid, as long as the distribution is ergodic and the sample size is sufficiently large. This admits a variant of the guarantees proved by \cite{van2002hoeffding}.

\section{FAST-DAD Distillation via Augmented Data}
\label{s:methods}

The augmentation distribution $q$ in Lemma \ref{lem:surrogate} could be naively produced by applying density estimation to the  data $X_n$, and then sampling from the learnt density. Unfortunately,
multivariate density estimation and generative modeling are at least
as difficult as the supervised learning problems
AutoML aims to solve~\cite{sugiyama2012density}. 
% This is hence inadvisable an \emph{automated} setting such as distillation for AutoML. 
It is however much easier to estimate $p(x^i|x^{-i})$, the univariate conditional of the feature $x^i$ given all the other features $x^{-i}$ in datum $x = (x^i, x^{-i})$. This suggests the following strategy which forms the crux of \method{}:
\begin{enumerate}[noitemsep,topsep=0pt,parsep=0pt,partopsep=0pt]
\item For all features $i$: estimate conditional distribution $p(x^i|x^{-i})$ using the training data.
\item Use all training data $x \in X_n$ as initializations for a Gibbs sampler \cite{geman1984stochastic}. That is, use each $x \in X_n$ to generate an MCMC chain  via: \  $\tilde{x}^i \sim p(x^i|x^{-i}), x^i \leftarrow \tilde{x}^i$.
\item Use the samples from all chains as additional data for distillation.
\end{enumerate}

We next describe these steps in detail but first let us see why this strategy can generate good augmented data. If our  conditional probability estimates $p(x^i|x^{-i})$ are accurate, the Gibbs sampler is guaranteed to converge to samples drawn from $p(x)$ regardless of the initialization~\cite{roberts1994simple}. In particular, initializing the sampler with data $x \in X_n$ ensures that it doesn't need time to `burn-in'; it starts immediately with samples from the correct distribution. Even if $p(x^i|x^{-i})$ is inaccurate (inevitable for small $n$), the sample $\tilde{x}$ will not deviate too far from $p(x)$ after a small number of Gibbs sampling steps (low bias), whereas using $\tilde{x} \sim q$ with an inaccurate $q$ would produce disparate samples.

\subsection{Maximum Pseudolikelihood Estimation via Self-Attention}
% It is cumbersome to learn all the conditionals $p(x^i|x^{-i})$ for different $i$ using separate models. Instead,
 A cumbersome aspect of the strategy outlined above is the need to model many conditional distributions $p(x^i|x^{-i})$ for different $i$. This would traditionally require many separate models. Here  
we instead propose a single self-attention architecture 
\cite{vaswani2017attention} with parameters $\theta$ that is trained to simultaneously estimate all conditionals via a pseudolikelihood objective~\cite{besag1977efficiency}:

\vspace*{-2em}
\begin{align}
  \label{eq:mle}
  \widehat{\theta} = \argmax_\theta \f{1}{n} \sum_{x \in X_n} \sum_{i=1}^d \log p(x^i\ | x^{-i}; \theta)
\vspace*{-1.5em}
\end{align}
For many models, maximum pseudolikelihood estimation produces asymptotically consistent parameter estimates, and  often is more computationally tractable than optimizing the likelihood~\cite{besag1977efficiency}. 
Our model takes as input $(x^1,\ldots,x^d)$ and simultaneously estimates the conditional distributions  $p(x^i|x^{-i}; \theta)$ for all features $i$ using a self-attention-based encoder.
As in Transformers, each encoder layer consists of a multi-head self-attention mechanism and a feature-wise feedforward block~\cite{vaswani2017attention}. % The former amounts to a linear combination of the embedding of each input feature $x^i$ with other features that are similar to $x^i$ getting a larger weight.
Self-attention helps this model gather relevant information from $x^{-i}$ needed for modeling $x^i$.
% Our encoder is similar to that of the Transformer~\cite{vaswani2017attention}, we also use multi-head self-attention, layer normalization and position encoding.% but we do not have a decoder.

Each conditional is parametrized as a mixture of Gaussians % with diagonal covariance 
$p(x^i|x^{-i}; \theta) = \sum_{k=1}^K \lambda_k N(x^i; \mu_k, \sigma_k^2)$, where $\lambda_k, \mu_k, \sigma_k$ depend on $x^{-i}$ and are output by topmost layer of our encoder after processing $x^{-i}$. 
Categorical features are numerically represented using dequantization~\cite{UriaRNADE}.
To condition on $x^{-i}$ in a mini-batch (with $i$ randomly selected per mini-batch), we mask the values of $x^i$ to omit all information about the corresponding feature value (as in \cite{devlin2019bert}) and also mask all self-attention weights for input dimension $i$; this amounts to performing stochastic gradient descent on the objective in~\cref{eq:mle} across both samples and their individual features. We thus have an efficient way to compute any of these conditional distributions with one forward pass of the model. While this work utilizes self-attention, our proposed method can work with any efficient estimator of $p(x^i| x^{-i})$ for $i=1,\dots,d$.

\heading{Relation to other architectures}
Our approach can be seen as an extension of the mixture density
network~\cite{bishop1994mixture}, which can model arbitrary
conditional distributions, but not all conditionals simultaneously as
enabled by our use of masked self-attention with the pseudolikelihood
objective. It is also similar to TraDE~\cite{fakoor2020trade}:
however, their auto-regressive model requires imposing an arbitrary
ordering of the features. Since self-attention
is permutation-invariant \cite{lee2018set}, our pseudolikelihood model
is desirably insensitive to the order in which features happen to be
recorded as table columns.
Our use of masked self-attention shares many similarities with BERT \cite{devlin2019bert},
where the  goal is typically representation learning or text
generation \cite{wang2019bert}. In contrast, our method is designed
for data that lives in tables. We need to estimate the conditionals
$p(x^i| x^{-i};\theta)$ very precisely as they are used to sample
continuous values; this is typically not necessary for text models.

\subsection{Gibbs Sampling from the Learnt Conditionals}

%When estimating
% graphical models, we have the same problem when computing the gradient
% of the log-partition function \cite{wainwright2008graphical}. It might
% take a long time and even worse, whenever the density is a poor
% fit for the data, all the work is of limited use since it doesn't
% affect the direction of the gradient too much.

We adopt the following procedure to draw Gibbs samples $\widetilde{x}$ to augment our training data: The sampler is initialized at some training example $x \in X_n$ and a random ordering of the features is selected (with different orderings used for different Gibbs chains started from different training examples). We cycle through the features and in each step replace the value of one feature in $\widetilde{x}$, say $\widetilde{x}^i$, using its conditional distribution given all the other variables, i.e. $p(x^i \mid x^{-i}; \widehat{\theta})$.
After every feature has been resampled, we say one \emph{round} of Gibbs sampling is complete, and proceed onto the next round by randomly selecting a new feature-order to follow in subsequent Gibbs sampling steps.

A practical challenge in Gibbs sampling is that a poor choice of  initialization may require many burn-in steps to produce reasonable samples. Suppose for the following discussion that our pseudolikelihood estimator and its learnt conditionals are accurate.
We can use a strategy inspired by Contrastive
Divergence \cite{hinton2002training} and initialize
the sampler at $x \in
X_n$ and take a few (often only one) Gibbs sampling steps. This strategy is effective; we need not wait for the sampler to burn in because it is initialized at (or close to) the true distribution itself. This is  
seen in~\cref{fig:toy} where we compare samples from the true distribution and Gibbs samples (taken with respect to conditional estimates from our self-attention network) starting from an arbitrary initialization vs.\ initialized at $X_n$. 

\begin{figure}[!htpb]
  \centering
  \begin{minipage}[c]{0.5\textwidth}
    \includegraphics[width=\textwidth]{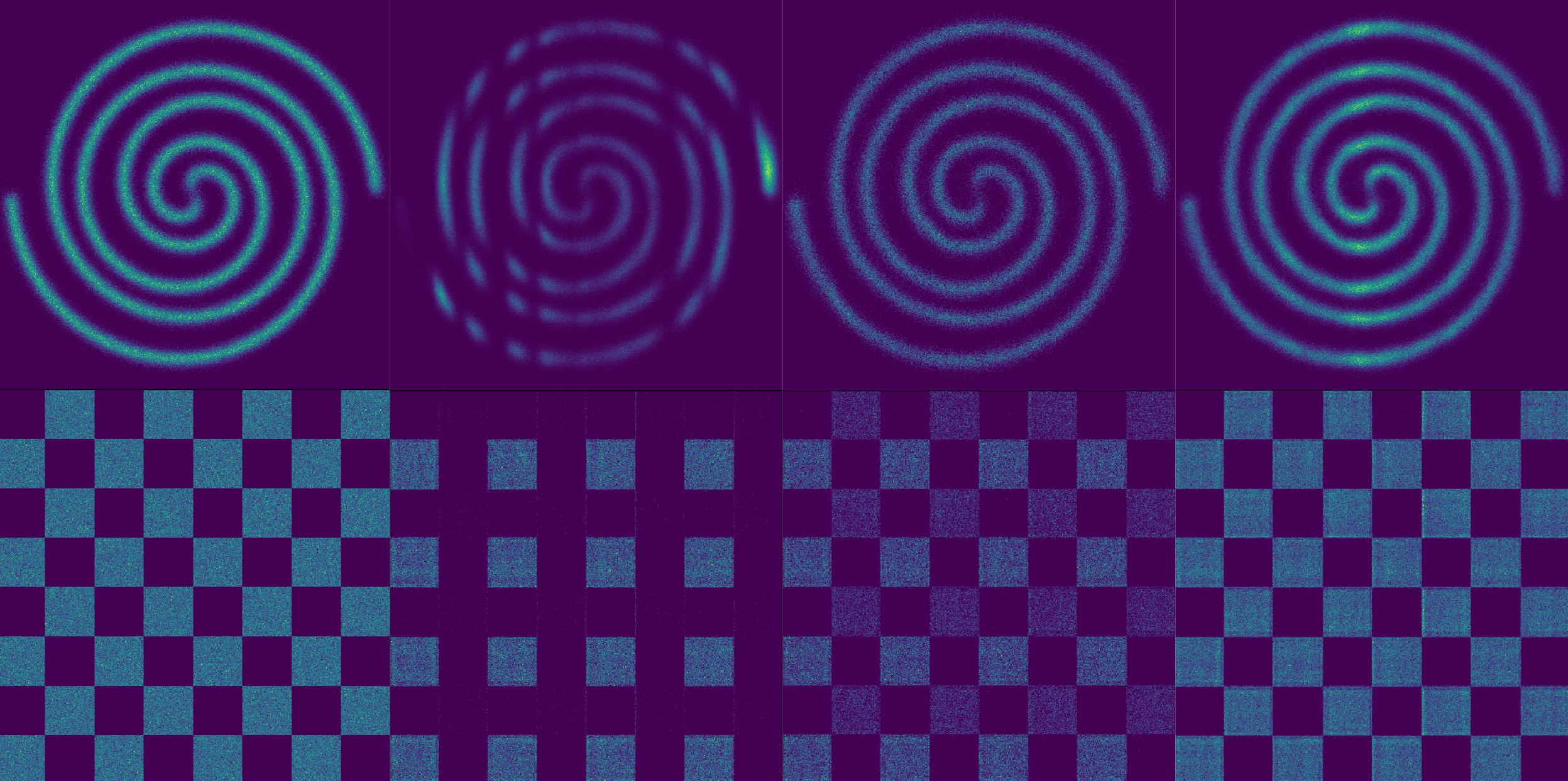}
  \end{minipage}\hspace{0.1in}
  \begin{minipage}[c]{0.45\textwidth}
    \caption{\tbf{Initialization of the Gibbs sampler.}
    From left to right: original training data, samples obtained from one
    round of Gibbs sampling with random initialization after fitting the self-attention network, samples
    obtained after multiple rounds of Gibbs sampling 
    (10 for the spiral, 100 for the checkerboard density) with random initialization, and samples
    obtained from one Gibbs sampling round when initializing via $X_n$. The densities were
    generated from examples in~\citet{nash2019arg}.}
  \label{fig:toy}
\end{minipage}
\end{figure}

For distillation, we expect this sampling strategy to produce   better augmented data. 
The number of Gibbs sampling steps provides fine-grained
control over the sample fidelity and diversity of the resulting dataset used in distillation.
Recall that the student will be trained over
$X_n \cup X_m'$ in practice. When our estimates of
$p(x^i|x^{-i})$ are accurate, it is desirable to produce
$X_m'$ only after a large number of Gibbs steps, as the
bias in \cref{eq:surrogate} will remain low and we would like to
ensure the $X_m'$ are more statistically independent from
$X_n$. With worse
estimates of $p(x^i|x^{-i})$, it is better to produce
$X_m'$ after only a few Gibbs steps to ensure lower bias in
\cref{eq:surrogate}, but the lack of burn-in implies 
$X_m'$ are not independent of $X_n$ and may thus be less useful to the student during distillation. We dig deeper into this phenomenon (for the special case of $m=n$) in the following theorem. 
\vspace*{-0.2em}
\begin{theorem}[Refinement of Lemma 1]
\label{thm:refinement}
Under the assumptions of Lemma 1, suppose the student $g^*$ minimizes $D_{\trm{emp}}(f, g, X_n \cup X_n')$ where $X_n'$ are $n$ samples drawn after $k$ steps of the Gibbs sampler initialized at $X_n$. Then there exist constants $V,c, \delta > 0$ such that with probability $\geq 1-\delta$:
\begin{align}
D(f, g^*, p) \le 
    D_{\trm{emp}}(f, g^*, X_n \cup X_n') + \sqrt{\f{4 V(c + \Delta_k) - \log \delta}{n}} + \Delta_k
\end{align}
$\Delta_k = \norm{T^k_q p - p}_{\trm{TV}}$ is the total-variation norm between $p$ and $T^k_q p$ (the distribution of Gibbs samples after $k$ steps), where $q$ denotes the steady-state distribution of the Gibbs sampler.
\end{theorem}
\vspace*{-0.5em}
%
% This result builds upon generalization bounds across multiple tasks~\cite{baxter2000model} and convergence rates for Markov chain-based samplers~\cite{wang2014convergence} (proof in~\cref{s:proof_refinement}). Note that $\Delta_k \to \norm{T^k_q p-q}_{\trm{TV}}$ as $k \to \infty$, and thus when $q$ is inaccurate (e.g.\ if our pseudolikelihood estimator is fit to limited data), we should  use Gibbs samples from a smaller number of steps.
The proof (in~\cref{s:proof_refinement}) is based on multi-task generalization bounds~\cite{baxter2000model} and MCMC mixing rates~\cite{wang2014convergence}.  Since $\Delta_k \to \norm{T^k_q p-q}_{\trm{TV}}$ as $k \to \infty$, we should use Gibbs samples from a smaller number of steps when $q$ is inaccurate (e.g.\ if our pseudolikelihood estimator is fit to limited data).

\subsection{Training the Student with Augmented Data}

While previous distillation works focused only on particular tasks
\cite{bucilua2006compress,hinton2015distilling}, we consider
the range of regression
%, binary and multiclass
and classification tasks. Our overall
approach is the same for each problem type:
\begin{enumerate}[noitemsep,topsep=0pt,parsep=0pt,partopsep=0pt]
    \item Generate a set of augmented samples $X_m' = \cbr{x_k'}_{k=1,\ldots,m}$.
    \item Feed the samples $X_m'$ as inputs to the teacher model to
      obtain predictions $Y_m'$, which are the predicted class
      probabilities in classification (rather than hard class
      labels), and predicted scalar values in regression.
    \item Train each student model on the augmented dataset
      $(X_n,Y_n) \cup (X_m', Y_m')$.
\end{enumerate}
In the final step, our student model is fit to a combination of both
true labels from the data $y$ as well as as augmented labels $y'$ from
the teacher, where  $y'$ is of different form than $y$ in
classification (predicted probabilities rather than predicted
classes). For binary classification tasks, we employ the Brier
score~\cite{brier1950verification} as our loss function for all
students, treating both the probabilities assigned to the positive
class by the teacher and the observed $\{0,1\}$ labels as continuous
regression targets for the student model. The same strategy was
employed by \citet{bucilua2006compress} and it slightly outperformed
our alternative multiclass-strategy in our binary classification
experiments. We handle multiclass classification in a manner specific to different
types of models, avoiding cumbersome students that maintain a separate
model for each class (c.f.\ one-vs-all). Neural network students are
trained using the cross-entropy loss which can
be applied to soft labels as well. Random forest
students can utilize multi-output decision trees
\cite{segal2011multivariate} and thus be trained as native
multi-output regressors against targets which are one-hot-encoded
class labels in the real data and teacher-predicted probability
vectors in the augmented data.
Boosted tree models are similarly used to predict vectors with one
dimension per class, which are then passed through a softmax
transformation; the cross entropy loss is
minimized via gradient boosting in this case.
% \vspace*{-1em}

\section{Experiments}

\heading{Data}
We evaluate various methods on $30$ %todo
datasets (\cref{tab:datasets}) spanning regression tasks from the UCI ML Repository and binary/multi classification tasks from OpenML, which are included in popular deep learning and AutoML  benchmarks  \citep{bayesdl,lakshminarayanan2017simple,jain2019maximizing,gijsbers2019open,truong2019towards, erickson2020autogluon}.
% TODO 
% (details in~\cref{sec:experimentdetails}).
%Stemming from the UCI ML Repository, the regression datasets have become a standard benchmark in Bayesian deep learning  \cite{bayesdl}, while the classification datasets stem from OpenML and were used in previous AutoML benchmarks   \citep{gijsbers2019open,truong2019towards,erickson2020autogluon}.
To facilitate comparisons on a meaningful scale across datasets, we evaluate methods on the provided test data based on either their accuracy in classification, or percentage of variation explained (= $R^2 \cdot 100$) in regression. The training data are split into training/validation folds (90-10), and only the training fold is used for augmentation (validation data keep their original labels for use in model/hyper-parameter selection and early-stopping).

\heading{Setup}
We adopt AutoGluon as our teacher as this system has demonstrated higher accuracy than most other AutoML frameworks and human data science teams \cite{erickson2020autogluon}. AutoGluon is fit to each training dataset for up to 4 hours with the \texttt{auto\_stack} option which boosts accuracy via extensive model ensembling (all other arguments left at their defaults).
The most accurate ensembles produced by AutoGluon often contain over 100 individual models trained via a combination of multi-layer stacking with repeated 10-fold bagging and the use of multiple hyperparameter values \cite{erickson2020autogluon}. Each model trained by AutoGluon is one of: (1) Neural Network (NN), (2) CatBoost, (3) LightGBM, (4) Random Forest (RF), (5) Extremely Randomized Trees, and (6) K-Nearest Neighbors.

We adopt the most accurate AutoGluon ensemble (on the validation data) as the teacher model. We use models of types (1)-(4) as students, since these are more efficient than the others and thus more appropriate for distillation. These are also some of the most popular types of models among today's data scientists \cite{kaggletrends}.
We consider how well each individual type of model performs under different training strategies, as well as the overall performance achieved with each strategy after a model selection step in which the best individual model on the validation data (among all 4 types) is used for prediction on the test data. This \emph{Selected} model reflects how machine learning is operationalized in practice. All candidate student models (as well as the BASE models) of each type share the same hyper-parameters and are expected to have similar size and inference latency.

% \vspace*{-0.1em}
\subsection{Distillation Strategies}
% \vspace*{-0.2em}

\begin{figure}[!htpb]
\centering
\captionsetup[subfigure]{justification=centering}
\begin{subfigure}[t]{0.44\textwidth}
\includegraphics[width=\textwidth]{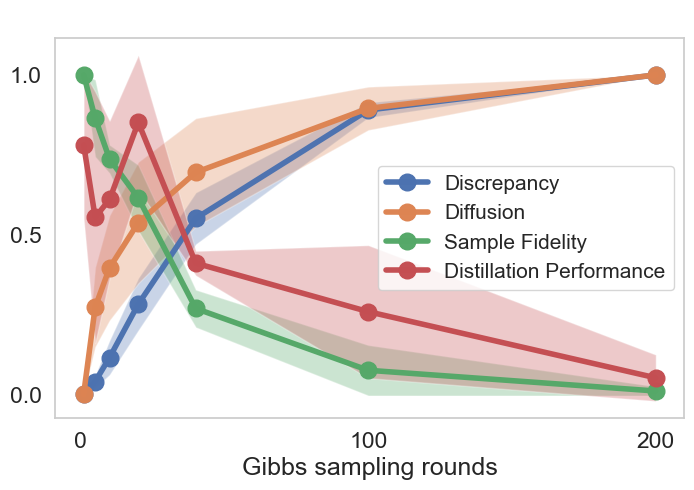}
\caption{}
\label{fig:gibbs_trend}
\end{subfigure}
\hspace{0.1in}
\begin{subfigure}[t]{0.40\textwidth}
\centering
\includegraphics[width=\textwidth]{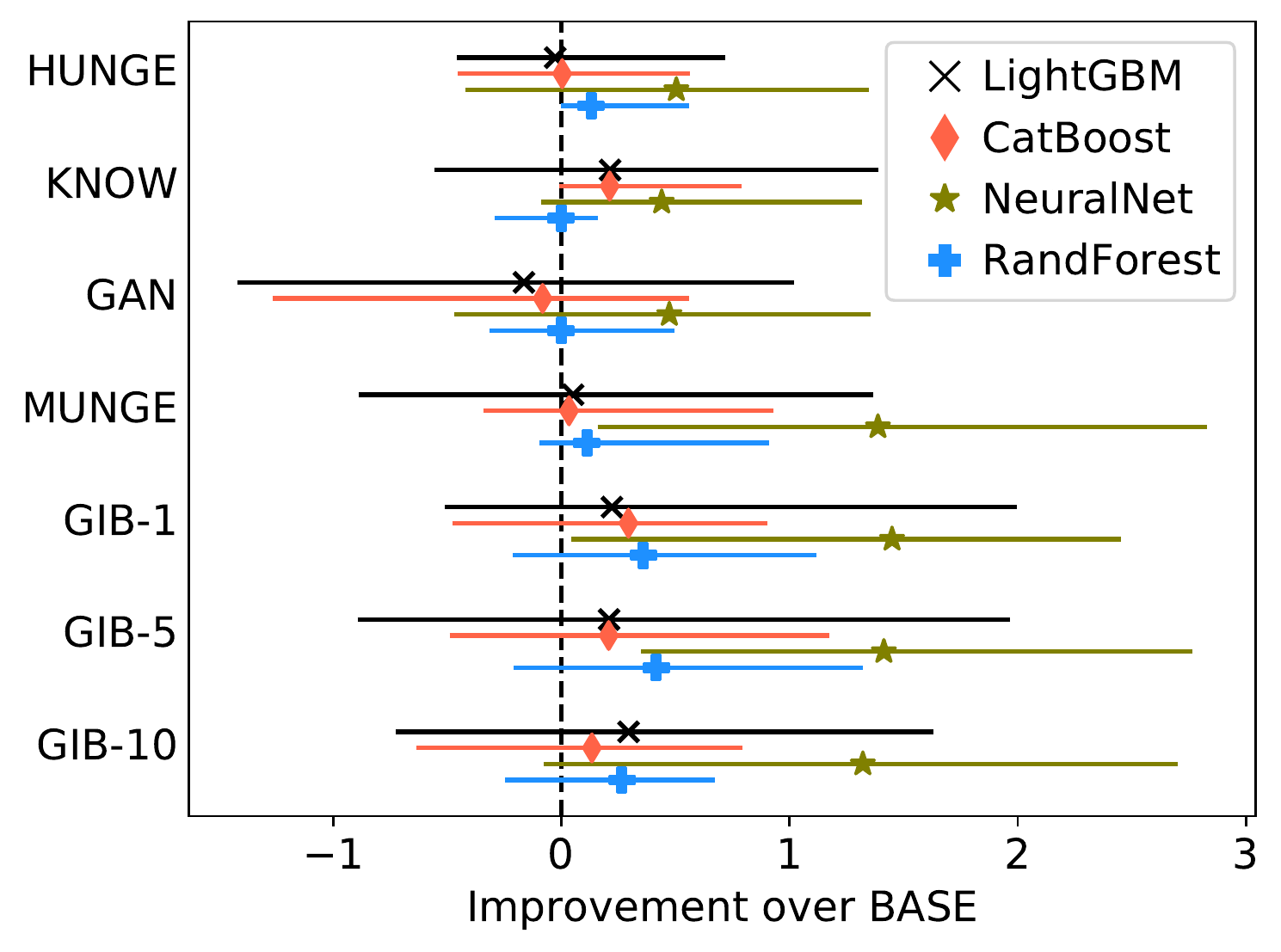}
\caption{}
\label{fig:improvements}
\end{subfigure}
\caption{\tbf{\cref{fig:gibbs_trend} Normalized metrics evaluated on samples from various Gibbs rounds (averaged across 3 datasets).} 
\emph{Sample Fidelity} measures how well a random forest discriminator can distinguish between (held out) real and Gibbs-sampled data. \emph{Diffusion} is the average Euclidean distance between each Gibbs sample and the datum from which its Markov chain was initialized. \emph{Discrepancy} is the Maximum Mean Discrepancy \cite{gretton2012kernel} between the Gibbs samples and the training data; it measures both how well the samples approximate $p$ as well as how distinct they are from data $X_n$. \emph{Distillation Performance} is the test accuracy of student models trained on the augmented data (averaged over our 4 model types).
The diversity of the overall dataset used for distillation grows with increased discrepancy/diffusion, while this overall dataset more closely resembles the underlying data-generating distribution with increased sample fidelity (lower bias).
The discriminator's accuracy ranges between $[0.49, 0.90]$ for these datasets.
% The discriminator's accuracy in distinguishing real vs. sampled data ranges between $[0.49, 0.71]$, % for Sylvine,
% $[0.54, 0.67]$, % for Phoneme, 
% $[0.58, 0.90]$ % for Protein.
% for all datasets.
\tbf{\cref{fig:improvements} Percentage points improvement over the BASE model produced by each distillation method} for different model types (change in: accuracy for classification, explained variation for regression). As the improvements contain outliers/skewness, we show the median change across all datasets (dots) and the corresponding interquartile range (lines).
}
\label{fig:mixed}
% \vspace*{-0.1in}
\end{figure}

We compare our FAST-DAD Gibbs-augmented distillation technique with the following methods.

\textbf{TEACHER}: The model ensemble produced by AutoGluon fit to the training data. This is adopted as the teacher in all distillation strategies we consider.
\textbf{BASE}: Individual base models fit to the original training data. 
% \\

\textbf{KNOW}: \emph{Knowledge distillation} proposed by \citet{hinton2015distilling}, in which we train each student model on the original training data, but with labels replaced by predicted probabilities from the teacher which are smoothed and nudged toward the original training labels (no data augmentation). 
% \\

\textbf{MUNGE}: The technique proposed by \citet{bucilua2006compress} to produce augmented data for model distillation, where the augmented samples are intended to resemble the underlying feature distribution. MUNGE augmentation may be viewed as a few steps of Gibbs sampling, where $p(x^i|x^{-i})$ is estimated by first finding near neighbors of $x$ in the training data and subsequently sampling from the (smoothed) empirical distribution of their $i^{\text{th}}$ feature \cite{owen1987nonparametric}.
% While MUNGE is conceptually simple, it requires nearest neighbor search with an appropriate distance metric.
% \\

%
\iffalse
\textbf{SMG}: We propose \emph{Simplified Munge}, a more straightforward/efficient variant of MUNGE in which a random subset of feature values are swapped between random training examples (not solely between near neighbors) with the same Gaussian additive noise used in MUNGE. Just as manifold-Mixup is more faithful to the training distribution than the popular Mixup augmentation strategy used in computer vision, MUNGE samples should better resemble the joint feature distribution than SMG samples (although as previously discussed this may not imply superior performance in our distillation context).
\fi
%
\textbf{HUNGE}: To see how useful the teacher's learned label-distributions are to the student, we apply a \emph{hard} variant of MUNGE. Here we produce MUNGE-augmented samples that receive hard class predictions from the teacher as their labels rather than the teacher's predicted probabilities that are otherwise the targets in all other distillation strategies (equivalent to MUNGE for regression).
% \\

%\textbf{GAN}: A recent adaptation of conditional generative-adversarial networks for augmenting tabular data~\citep{xu2019modeling}. This method has been shown to produce better samples than other deep generative models like the basic GAN architectures used by \citet{liu2018teacher}.  Like our model, this GAN is trained on the training set and then used to generate augmented $x$ samples for the student model, whose labels are the predicted probabilities output by the teacher. Unlike our Gibbs sampling strategy, it is difficult to control how similar samples from the GAN should be to the training data.
\textbf{GAN}:  The technique proposed by \citet{xu2019modeling} for augmenting tabular data using conditional deep generative adversarial networks (GANs); this performs better than other GANs~\cite{liu2018teacher}. Like our model, this GAN is trained on the training set and then used to generate augmented $x$ samples for the student model, whose labels are the predicted probabilities output by the teacher. Unlike our Gibbs sampling strategy, it is difficult to control how similar samples from the GAN should be to the training data.

We also run our Gibbs sampling data augmentation strategy generating samples after various numbers of Gibbs sampling rounds (for example, \textbf{GIB-5} indicates 5 rounds were used to produce the augmented data).  Under each augmentation-based  strategies, we add $m$ synthetic datapoints to the training set for the student, where $m= 10 \times$ the number of original training samples (up to at most $10^6$).

% To provide an external frame of reference on accuracy ranges for each dataset, we also apply two other popular AutoML frameworks: \textbf{H2O} \cite{H2O} and \textbf{AutoSklearn} \cite{feurer2019auto}, which have been shown to outperform other AutoML tools   \cite{truong2019towards,guyon2019analysis}.

% \vspace*{-1em}
% \vspace*{-0.1em}
\subsection{Analysis of the Gibbs Sampler}
% \vspace*{-0.1em}

%\begin{wrapfigure}{r}{0.30\textwidth}

To study the behavior of our Gibbs sampling procedure, we evaluate it on a number of different criteria (see \cref{fig:gibbs_trend} caption).
\cref{fig:gibbs_trend} depicts how the distillation dataset's overall diversity increases with additional rounds of Gibbs sampling. Fortuitously, we do not require a large number of Gibbs sampling rounds to obtain the best distillation performance and can thus efficiently generate augmented data. Running the Gibbs sampling for longer is ill-advised as its stationary distribution appears to less closely approximate $p$ than intermediate samples from a partially burned-in chain; this is likely due to the fact that we have limited data to fit the self-attention network.

\begin{table*}[!bth]
\renewcommand{\arraystretch}{1.4}
\centering
\caption{%Comparing the \emph{Select}-model produced under various training strategies.
\tbf{Average ranks/performance achieved by the \emph{Selected} model} under each training strategy across the datasets from each prediction task. Performance is test accuracy for classification  or percentage of variation explained for regression, and we list $p$-values for the one-sided test of whether each strategy $\ge$ BASE.
}
\label{tab:selavgs}
\vspace*{0.2em}
\centering
\begin{footnotesize}
\resizebox{0.9 \textwidth}{!}{
\begin{tabular}{lcccgggccc}
\toprule
\textbf{Strategy} & \textbf{Rank} & \textbf{Accuracy} & \textbf{$\bm{p}$} 
& \cellcolor{white} \textbf{Rank} & \cellcolor{white} \textbf{Accuracy} & \cellcolor{white} \textbf{$\bm{p}$} &
\textbf{Rank} & \textbf{Accuracy} & \textbf{$\bm{p}$} \\ 
% [0.15em]
\midrule
% \\[-0.7em]
          BASE &               6.888 &               88.63 &                    - &               5.791 &               82.85 &                    - &                 7.777 &                 80.80 &                      - \\
         HUNGE &                 5.0 &               88.99 &                0.092 &               5.541 &               83.57 &                0.108 &                 7.666 &                 81.04 &                  0.350 \\
          KNOW &               6.555 &               88.49 &                0.712 &                5.25 &               83.89 &                0.072 &                 5.555 &                 81.39 &                  0.275 \\
           GAN &               6.666 &               88.65 &                0.450 &               6.708 &               83.17 &                0.250 &                 6.055 &                 82.26 &                  0.069 \\
         MUNGE &               5.444 &               88.88 &                0.209 &               5.083 &               83.72 &                0.126 &                 4.333 &                 82.80 &                  0.007 \\
         GIB-1 &               3.777 &               \textbf{89.35} &                0.025 &               \textbf{3.708} &               \textbf{84.21} &                \textbf{0.051} &                 \textbf{3.277} &                 \textbf{82.88} &                  \textbf{0.005} \\
         GIB-5 &               \textbf{3.333} &               89.25 &                \textbf{0.004} &               5.375 &               84.04 &                0.098 &                 3.388 &                 82.76 &                  0.010 \\
        GIB-10 &               4.777 &               89.09 &                0.044 &               4.958 &               83.74 &                0.087 &                 4.222 &                 82.64 &                  0.010 \\
       TEACHER &               2.555 &               90.10 &                0.036 &               2.583 &               84.40 &                0.019 &                 2.722 &                 83.84 &                  0.018 \\
\hline
& \multicolumn{3}{c}{\textbf{Regression Problems}} & 
 \multicolumn{3}{c}{\cellcolor{LightGray} \textbf{Binary Classification}} &
\multicolumn{3}{c}{\textbf{Multiclass Classification}} 
\end{tabular}

}
\end{footnotesize}
% \end{center}
\end{table*}

% \vspace*{-0.1em}
\subsection{Performance of Distilled Models}
% \vspace*{-0.1em}

\cref{tab:selavgs} and \cref{fig:improvements} demonstrate that our Gibbs augmentation strategy produces far better resulting models than any of the other strategies.
% Fortuitously running many Gibbs rounds causes the augmented data to converge toward the Transformer-estimated distribution whose errors lead to worse performance in the downstream distillation task.
\cref{tab:selaccs} shows the only datasets where Gibbs augmentation fails to produce better models than the BASE training strategy are those where the teacher ensemble fails to outperform the best individual BASE model (so  little can be gained from distillation period).
As expected according to~\citet{hinton2015distilling}:
% In line with \citet{hinton2015distilling}:
KNOW helps in classification but not regression, and HUNGE fares worse than MUNGE on multiclass problems where its augmented hard class-labels fail to provide students with the teacher's \emph{dark knowledge}. As previously observed \cite{bucilua2006compress}, MUNGE greatly improves the performance of neural networks, but  provides less benefits for the other model types than augmentation via our Gibbs sampler.  Overparameterized deep networks tend to benefit from distillation more than the tree models in our experiments (although for numerous datasets  distilled tree models are still \emph{Selected} as the best model to predict with). While neural nets trained in the  standard fashion are usually less accurate than trees for tabular data, \method{} can boost their performance above that of trees, a goal other research has struggled to reach \cite{biau2019neural,saberian2019gradient,popov2019neural,ke2019deepgbm,tabnn}.

\cref{fig:improvements-hero,fig:binaryacclat,fig:automlregressmulti} depict the (normalized/raw) accuracy and inference-latency of our distilled models (under the GIB-1 strategy which is superior to others), compared with both the teacher (AutoGluon ensemble), as well as ensembles produced by H2O-AutoML \cite{H2O} and AutoSklearn \cite{feurer2019auto}, two  popular AutoML frameworks that have been shown to outperform other AutoML tools   \cite{truong2019towards,guyon2019analysis}. On average, the \emph{Selected}  individual model under standard training (BASE) would be outperformed by these AutoML ensembles, but surprisingly, our distillation approach produces \emph{Selected} individual models that are both more accurate and over 10$\times$ more efficient than H2O and AutoSklearn. 
% Note that the \emph{Selected} BASE models are worse than individual RF/LightGBM models, presumably due to overfitting of the validation set via the early-stopping criterion in NN/CatBoost, which appears to be mitigated by distillation with augmented data.
In multiclass classification, our distillation approach also confers significant accuracy gains over standard training. The resulting individual \emph{Selected} models come close to matching the accuracy of H2O/AutoSklearn while offering much lower latency, but gains may be limited since the AutoGluon teacher appears only marginally more accurate than H2O/AutoSklearn in these multiclass problems.

% Improvements per model type figure:
%\begin{figure}[tb!] \centering
%\begin{tabular}{cc}
%\includegraphics[width=0.47\textwidth]{figures/improvements_per_modeltype.pdf} & \includegraphics[width=0.47\textwidth]{figures/binary_accuracy_vs_latency.pdf}\\
%\textbf{(A)} & \textbf{(B)}
%\end{tabular}
%\caption{\textbf{(A)} Percentage points improvement over the BASE model produced by each distillation method for different model types (change in: accuracy for classification, explained variation   for regression). As the improvements contain outliers/skewness, we show the median change across all datasets (dots) and the corresponding  interquartile range (lines). \textbf{(B)} Test accuracy vs latency (seconds per 1000 predictions) of individual models and AutoML ensembles (averaged over the binary classification datasets, other problem types in  \cref{fig:automlregressmulti}). The GIB-1 and BASE dots show performance of \emph{Selected} model (out of the 4 types) on each dataset.
%}
%\label{fig:improvements}
%\end{figure}

%While distillation has become a popular research subject in deep learning, existing work is primarily focused on computer vision applications \cite{cho2019efficacy,mirzadeh2019improved, yang2020knowledge}. These image datasets tend to be massive and thus augmentation is not as critical as it is in the tabular setting.

\begin{wrapfigure}{r}{0.42\textwidth}
\vspace*{-2.5mm}
\centering
    \includegraphics[width=0.42\textwidth]{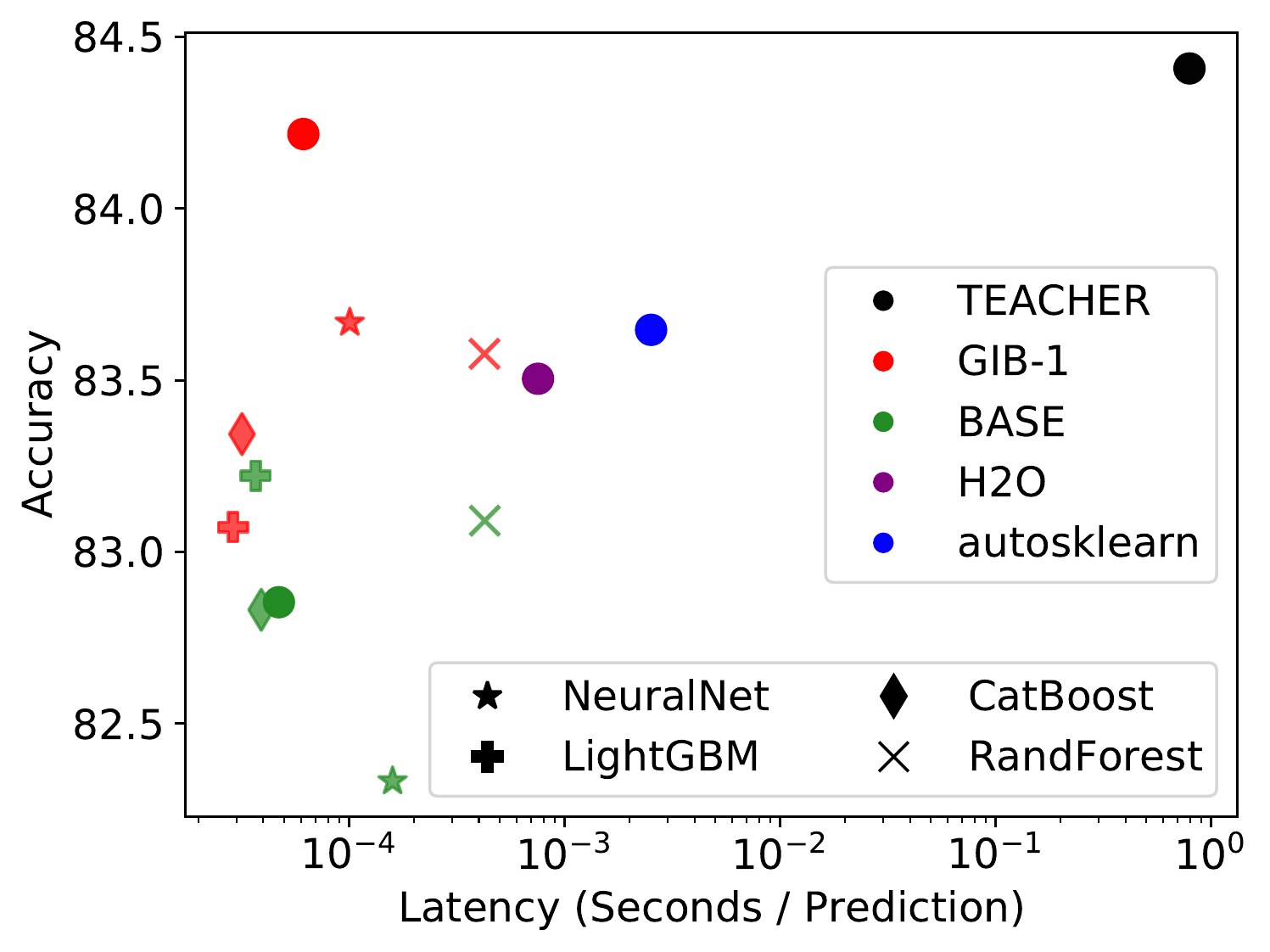}
    \caption{\tbf{Raw test accuracy vs.\ speed 
      of individual models and AutoML ensembles, averaged over binary classification datasets}. TEACHER denotes the performance of AutoGluon; H2O and autosklearn represent the respective AutoML tools. GIB-1 indicates the results of \method{} after 1 round of Gibbs sampling. BASE denotes the student model fit on the original data.  GIB-1/BASE dots represent the  \emph{Selected} model. 
      }  
    \label{fig:binaryacclat}
    \vspace{10mm}
\end{wrapfigure}

\section{Discussion}
% \vspace*{-0.1em}
Our goal in this paper is to build small, fast models that can bootstrap off large, ensemble-based AutoML predictors via model distillation to perform better than they would if directly fit to the original data.
The key challenge arises when the data to train the student are limited. We propose to estimate the conditional distributions of all features via maximum pseudolikelihood with masked self-attention, and Gibbs sampling techniques to sample from this model for augmenting the data available to the student.
Our strategy neatly suits the application because it: (i) avoids multivariate  density estimation (pseudolikelihood only involves univariate conditionals), (ii) does not require separate models for each conditional (the self-attention model simultaneously computes all conditionals), (iii) is far more efficient than usual MCMC methods (by initializing the Gibbs sampler at the training data), (iv) allows control over the quality and diversity of the resulting augmented dataset (we can select samples from specific Gibbs rounds unlike from, say, a GAN).
% Gibbs sampling is an ideal fit in this application because it only requires estimating conditionals rather than a multivariate distribution, and we can initialize the Markov chains at training examples to control the sample-quality and diversity of the resulting augmented dataset. Just one round of Gibbs sampling is all we need to produce beneficial augmented data.
% distributional similarity between our augmented and training data.
% Extensive experiments on 30 datasets spanning classification and regression tasks show that our data augmentation strategy produces much better distillation performance than existing techniques.
We used the high-accuracy ensembles produced by AutoGluon's AutoML to improve standard boosted trees, random forests, and neural networks via distillation.
%, and have contributed our methodology back into the open-source AutoGluon codebase for others to utilize.

%% Scratch

\iffalse
% PLACEHOLDER TABLE
% Accuracy
\begin{table*}[!bth] % 1h accuracy all scores
\centering
% REMOVE IN FINAL VERSION: This contains results from models optimized on metrics besides accuracy.
\caption{Misclassification Error Rate ($= 1 -$ Accuracy) on test data produced by AutoML frameworks in the AutoML Benchmark (after training with 1h time limit).
The best performance among all AutoML frameworks is highlighted in bold, and failed runs are indicated by a cross. Note that all frameworks were optimized for the AUC objective (as in the original AutoML Benchmark).
}
\label{tab:rawlossopenmlaccuracy1h}
\vspace*{1em}

\begin{footnotesize}
\input{tables/openml_accuracy_1h_openml_alllosses_1h.tex}
\end{footnotesize}
% \end{center}
\vskip -0.1in
\end{table*}

% PLACEHOLDER TABLE
\begin{table*}[!bth] % 1h all AUC/LogLoss scores
\centering
\caption{Loss on test data produced by AutoML frameworks in our AutoML Benchmark (after training with 1h time limit).
The best performance among all AutoML frameworks is highlighted in bold, and failed runs are indicated by a cross. Binary classification datasets indicate AUC error, and Multiclass classification datasets indicate logloss error.
}
\label{tab:xx}
\vspace*{1em}

 \begin{footnotesize}
\input{tables/openml_alllosses_1h.tex}
 \end{footnotesize}
% \end{center}
\vskip -0.1in
\end{table*}
\fi

\clearpage
{\small
\bibliographystyle{IEEEtranN}% {unsrt}
\bibliography{distill}

% Generated by IEEEtranN.bst, version: 1.14 (2015/08/26)
\begin{thebibliography}{51}
\providecommand{\natexlab}[1]{#1}
\providecommand{\url}[1]{#1}
\csname url@samestyle\endcsname
\providecommand{\newblock}{\relax}
\providecommand{\bibinfo}[2]{#2}
\providecommand{\BIBentrySTDinterwordspacing}{\spaceskip=0pt\relax}
\providecommand{\BIBentryALTinterwordstretchfactor}{4}
\providecommand{\BIBentryALTinterwordspacing}{\spaceskip=\fontdimen2\font plus
\BIBentryALTinterwordstretchfactor\fontdimen3\font minus
  \fontdimen4\font\relax}
\providecommand{\BIBforeignlanguage}[2]{{%
\expandafter\ifx\csname l@#1\endcsname\relax
\typeout{** WARNING: IEEEtranN.bst: No hyphenation pattern has been}%
\typeout{** loaded for the language `#1'. Using the pattern for}%
\typeout{** the default language instead.}%
\else
\language=\csname l@#1\endcsname
\fi
#2}}
\providecommand{\BIBdecl}{\relax}
\BIBdecl

\bibitem[Erickson et~al.(2020)Erickson, Mueller, Shirkov, Zhang, Larroy, Li,
  and Smola]{erickson2020autogluon}
N.~Erickson, J.~Mueller, A.~Shirkov, H.~Zhang, P.~Larroy, M.~Li, and A.~Smola,
  ``{AutoGluon-Tabular}: Robust and accurate {AutoML} for structured data,''
  \emph{arXiv preprint arXiv:2003.06505}, 2020.

\bibitem[Feurer et~al.(2019)Feurer, Klein, Eggensperger, Springenberg, Blum,
  and Hutter]{feurer2019auto}
M.~Feurer, A.~Klein, K.~Eggensperger, J.~T. Springenberg, M.~Blum, and
  F.~Hutter, ``Auto-sklearn: efficient and robust automated machine learning,''
  in \emph{Automated Machine Learning}.\hskip 1em plus 0.5em minus 0.4em\relax
  Springer, 2019, pp. 113--134.

\bibitem[Cortes et~al.(2017)Cortes, Gonzalvo, Kuznetsov, Mohri, and
  Yang]{cortes2017adanet}
C.~Cortes, X.~Gonzalvo, V.~Kuznetsov, M.~Mohri, and S.~Yang, ``{A}da{N}et:
  Adaptive structural learning of artificial neural networks,'' in
  \emph{Proceedings of the 34th International Conference on Machine Learning},
  vol.~70.\hskip 1em plus 0.5em minus 0.4em\relax PMLR, 2017, pp. 874--883.

\bibitem[Bucilua et~al.(2006)Bucilua, Caruana, and
  Niculescu-Mizil]{bucilua2006compress}
C.~Bucilua, R.~Caruana, and A.~Niculescu-Mizil, ``Model compression,'' in
  \emph{Proceedings of the 12th ACM SIGKDD international conference on
  Knowledge discovery and data mining}, 2006, pp. 535--541.

\bibitem[Hinton et~al.(2015)Hinton, Vinyals, and Dean]{hinton2015distilling}
G.~Hinton, O.~Vinyals, and J.~Dean, ``Distilling the knowledge in a neural
  network,'' \emph{NIPS Deep Learning and Representation Learning Workshop},
  2015.

\bibitem[Dietterich(2000)]{dietterich2000ensemble}
T.~G. Dietterich, ``Ensemble methods in machine learning,'' in
  \emph{International Workshop on Multiple Classifier Systems}.\hskip 1em plus
  0.5em minus 0.4em\relax Springer, 2000, pp. 1--15.

\bibitem[Pandey(2019)]{H2O}
\BIBentryALTinterwordspacing
P.~Pandey, ``{A Deep Dive into H2O’s AutoML},'' 2019. [Online]. Available:
  \url{http://www.h2o.ai/blog/a-deep-dive-into-h2os-automl/}
\BIBentrySTDinterwordspacing

\bibitem[Tang et~al.(2020)Tang, Shivanna, Zhao, Lin, Singh, Chi, and
  Jain]{tang2020understanding}
J.~Tang, R.~Shivanna, Z.~Zhao, D.~Lin, A.~Singh, E.~H. Chi, and S.~Jain,
  ``Understanding and improving knowledge distillation,'' \emph{arXiv preprint
  arXiv:2002.03532}, 2020.

\bibitem[Vidal et~al.(2020)Vidal, Pacheco, and Schiffer]{vidal2020born}
T.~Vidal, T.~Pacheco, and M.~Schiffer, ``Born-again tree ensembles,''
  \emph{arXiv preprint arXiv:2003.11132}, 2020.

\bibitem[Cho and Hariharan(2019)]{cho2019efficacy}
J.~H. Cho and B.~Hariharan, ``On the efficacy of knowledge distillation,'' in
  \emph{Proceedings of the IEEE International Conference on Computer Vision},
  2019, pp. 4794--4802.

\bibitem[Ba and Caruana(2014)]{BaCaruana14}
J.~Ba and R.~Caruana, ``Do deep nets really need to be deep?'' in
  \emph{Advances in Neural Information Processing Systems}, Z.~Ghahramani,
  M.~Welling, C.~Cortes, N.~D. Lawrence, and K.~Q. Weinberger, Eds.\hskip 1em
  plus 0.5em minus 0.4em\relax Curran Associates, Inc., 2014, pp. 2654--2662.

\bibitem[Urban et~al.(2017)Urban, Geras, Kahou, Aslan, Wang, Caruana, Mohamed,
  Philipose, and Richardson]{urban2016deep}
G.~Urban, K.~J. Geras, S.~E. Kahou, O.~Aslan, S.~Wang, R.~Caruana, A.~Mohamed,
  M.~Philipose, and M.~Richardson, ``Do deep convolutional nets really need to
  be deep and convolutional?'' in \emph{International Conference on Learning
  Representations}, 2017.

\bibitem[Mirzadeh et~al.(2019)Mirzadeh, Farajtabar, Li, and
  Ghasemzadeh]{mirzadeh2019improved}
S.-I. Mirzadeh, M.~Farajtabar, A.~Li, and H.~Ghasemzadeh, ``Improved knowledge
  distillation via teacher assistant: Bridging the gap between student and
  teacher,'' \emph{arXiv preprint arXiv:1902.03393}, 2019.

\bibitem[Yang et~al.(2020)Yang, Martinez, Bulat, and
  Tzimiropoulos]{yang2020knowledge}
J.~Yang, B.~Martinez, A.~Bulat, and G.~Tzimiropoulos, ``Knowledge distillation
  via adaptive instance normalization,'' \emph{arXiv preprint
  arXiv:2003.04289}, 2020.

\bibitem[Breiman and Shang(1996)]{breiman96}
L.~Breiman and N.~Shang, ``Born again trees,'' \url{ftp://ftp.
  stat.berkeley.edu/pub/users/breiman/BAtrees.ps}, 1996.

\bibitem[Liu et~al.(2020)Liu, Fusi, and Mackey]{liu2018teacher}
R.~Liu, N.~Fusi, and L.~Mackey, ``Teacher-student compression with generative
  adversarial networks,'' \emph{arXiv preprint arXiv:1812.02271}, 2020.

\bibitem[Wilbraham(1848)]{wilbraham1848certain}
H.~Wilbraham, ``On a certain periodic function,'' \emph{The Cambridge and
  Dublin Mathematical Journal}, vol.~3, pp. 198--201, 1848.

\bibitem[Vapnik(1998)]{vapnik1998statistical}
V.~Vapnik, \emph{Statistical Learning Theory}.\hskip 1em plus 0.5em minus
  0.4em\relax John Wiley \& Sons, 1998.

\bibitem[Sugiyama et~al.(2012)Sugiyama, Suzuki, and
  Kanamori]{sugiyama2012density}
M.~Sugiyama, T.~Suzuki, and T.~Kanamori, \emph{Density ratio estimation in
  machine learning}.\hskip 1em plus 0.5em minus 0.4em\relax Cambridge
  University Press, 2012.

\bibitem[Geman and Geman(1984)]{geman1984stochastic}
S.~Geman and D.~Geman, ``Stochastic relaxation, gibbs distributions, and the
  bayesian restoration of images,'' \emph{IEEE Transactions on pattern analysis
  and machine intelligence}, no.~6, pp. 721--741, 1984.

\bibitem[Roberts and Smith(1994)]{roberts1994simple}
G.~O. Roberts and A.~F. Smith, ``Simple conditions for the convergence of the
  gibbs sampler and metropolis-hastings algorithms,'' \emph{Stochastic
  processes and their applications}, vol.~49, no.~2, pp. 207--216, 1994.

\bibitem[Vaswani et~al.(2017)Vaswani, Shazeer, Parmar, Uszkoreit, Jones, Gomez,
  Kaiser, and Polosukhin]{vaswani2017attention}
A.~Vaswani, N.~Shazeer, N.~Parmar, J.~Uszkoreit, L.~Jones, A.~N. Gomez,
  L.~Kaiser, and I.~Polosukhin, ``Attention is all you need,'' in
  \emph{Advances in Neural Information Processing Systems}, 2017.

\bibitem[Besag(1977)]{besag1977efficiency}
J.~Besag, ``Efficiency of pseudolikelihood estimation for simple gaussian
  fields,'' \emph{Biometrika}, pp. 616--618, 1977.

\bibitem[Uria et~al.(2013)Uria, Murray, and Larochelle]{UriaRNADE}
B.~Uria, I.~Murray, and H.~Larochelle, ``{RNADE}: The real-valued neural
  autoregressive density-estimator,'' in \emph{Advances in Neural Information
  Processing Systems}, 2013.

\bibitem[Devlin et~al.(2019)Devlin, Chang, Lee, and Toutanova]{devlin2019bert}
J.~Devlin, M.-W. Chang, K.~Lee, and K.~Toutanova, ``Bert: Pre-training of deep
  bidirectional transformers for language understanding,'' in \emph{NAACL HLT},
  2019.

\bibitem[Bishop(1994)]{bishop1994mixture}
C.~M. Bishop, ``Mixture density networks,'' \emph{Neural Computing Research
  Group Report, Aston University}, 1994.

\bibitem[Fakoor et~al.(2020)Fakoor, Chaudhari, Mueller, and
  Smola]{fakoor2020trade}
R.~Fakoor, P.~Chaudhari, J.~Mueller, and A.~J. Smola, ``{TraDE}: Transformers
  for density estimation,'' \emph{arXiv preprint arXiv:2004.02441}, 2020.

\bibitem[Lee et~al.(2018)Lee, Lee, Kim, Kosiorek, Choi, and Teh]{lee2018set}
J.~Lee, Y.~Lee, J.~Kim, A.~R. Kosiorek, S.~Choi, and Y.~W. Teh, ``Set
  transformer: A framework for attention-based permutation-invariant neural
  networks,'' \emph{arXiv preprint arXiv:1810.00825}, 2018.

\bibitem[Wang et~al.(2019)Wang, Cho, and Scholar]{wang2019bert}
A.~Wang, K.~Cho, and C.~A.~G. Scholar, ``Bert has a mouth, and it must speak:
  Bert as a markov random field language model,'' in \emph{NAACL HLT}, 2019.

\bibitem[Hinton(2002)]{hinton2002training}
G.~E. Hinton, ``Training products of experts by minimizing contrastive
  divergence,'' \emph{Neural Computation}, vol.~14, no.~8, pp. 1771--1800,
  2002.

\bibitem[Nash and Durkan(2019)]{nash2019arg}
C.~Nash and C.~Durkan, ``Autoregressive energy machines,'' \emph{arXiv preprint
  arXiv:1904.05626}, 2019.

\bibitem[Baxter(2000)]{baxter2000model}
J.~Baxter, ``A model of inductive bias learning,'' \emph{Journal of artificial
  intelligence research}, vol.~12, pp. 149--198, 2000.

\bibitem[Wang et~al.(2014)Wang, Wu, et~al.]{wang2014convergence}
N.-Y. Wang, L.~Wu \emph{et~al.}, ``Convergence rate and concentration
  inequalities for gibbs sampling in high dimension,'' \emph{Bernoulli},
  vol.~20, no.~4, pp. 1698--1716, 2014.

\bibitem[Brier(1950)]{brier1950verification}
G.~W. Brier, ``Verification of forecasts expressed in terms of probability,''
  \emph{Monthly weather review}, vol.~78, no.~1, pp. 1--3, 1950.

\bibitem[Segal and Xiao(2011)]{segal2011multivariate}
M.~Segal and Y.~Xiao, ``Multivariate random forests,'' \emph{Wiley
  Interdisciplinary Reviews: Data Mining and Knowledge Discovery}, vol.~1,
  no.~1, pp. 80--87, 2011.

\bibitem[Mukhoti et~al.(2018)Mukhoti, Stenetorp, and Gal]{bayesdl}
J.~Mukhoti, P.~Stenetorp, and Y.~Gal, ``On the importance of strong baselines
  in bayesian deep learning,'' in \emph{Bayesian Deep Learning NeurIPS 2019
  Workshop}, 2018.

\bibitem[Lakshminarayanan et~al.(2017)Lakshminarayanan, Pritzel, and
  Blundell]{lakshminarayanan2017simple}
B.~Lakshminarayanan, A.~Pritzel, and C.~Blundell, ``Simple and scalable
  predictive uncertainty estimation using deep ensembles,'' in \emph{Advances
  in neural information processing systems}, 2017, pp. 6402--6413.

\bibitem[Jain et~al.(2020)Jain, Liu, Mueller, and Gifford]{jain2019maximizing}
S.~Jain, G.~Liu, J.~Mueller, and D.~Gifford, ``Maximizing overall diversity for
  improved uncertainty estimates in deep ensembles,'' in \emph{AAAI}, 2020.

\bibitem[Gijsbers et~al.(2019)Gijsbers, LeDell, Thomas, Poirier, Bischl, and
  Vanschoren]{gijsbers2019open}
P.~Gijsbers, E.~LeDell, J.~Thomas, S.~Poirier, B.~Bischl, and J.~Vanschoren,
  ``An open source {AutoML} benchmark,'' in \emph{ICML Workshop on Automated
  Machine Learning}, 2019.

\bibitem[Truong et~al.(2019)Truong, Walters, Goodsitt, Hines, Bruss, and
  Farivar]{truong2019towards}
A.~Truong, A.~Walters, J.~Goodsitt, K.~Hines, B.~Bruss, and R.~Farivar,
  ``Towards automated machine learning: Evaluation and comparison of automl
  approaches and tools,'' \emph{arXiv preprint arXiv:1908.05557}, 2019.

\bibitem[Bansal(2018)]{kaggletrends}
S.~Bansal, ``Data science trends on kaggle,''
  \url{https://www.kaggle.com/shivamb/data-science-trends-on-kaggle#5.-XgBoost-vs-Keras},
  2018.

\bibitem[Gretton et~al.(2012)Gretton, Borgwardt, Rasch, Sch{\"o}lkopf, and
  Smola]{gretton2012kernel}
A.~Gretton, K.~M. Borgwardt, M.~J. Rasch, B.~Sch{\"o}lkopf, and A.~Smola, ``A
  kernel two-sample test,'' \emph{Journal of Machine Learning Research},
  vol.~13, no. Mar, pp. 723--773, 2012.

\bibitem[Owen(1987)]{owen1987nonparametric}
A.~Owen, ``Nonparametric conditional estimation,'' \emph{OSTI.GOV Technical
  Report}, 1987.

\bibitem[Xu et~al.(2019)Xu, Skoularidou, Cuesta-Infante, and
  Veeramachaneni]{xu2019modeling}
L.~Xu, M.~Skoularidou, A.~Cuesta-Infante, and K.~Veeramachaneni, ``Modeling
  tabular data using conditional gan,'' in \emph{Advances in Neural Information
  Processing Systems}, 2019.

\bibitem[Biau et~al.(2019)Biau, Scornet, and Welbl]{biau2019neural}
G.~Biau, E.~Scornet, and J.~Welbl, ``Neural random forests,'' \emph{Sankhya A},
  vol.~81, no.~2, pp. 347--386, 2019.

\bibitem[Saberian et~al.(2019)Saberian, Delgado, and
  Raimond]{saberian2019gradient}
M.~Saberian, P.~Delgado, and Y.~Raimond, ``Gradient boosted decision tree
  neural network,'' \emph{arXiv preprint arXiv:1910.09340}, 2019.

\bibitem[Popov et~al.(2019)Popov, Morozov, and Babenko]{popov2019neural}
S.~Popov, S.~Morozov, and A.~Babenko, ``Neural oblivious decision ensembles for
  deep learning on tabular data,'' \emph{arXiv preprint arXiv:1909.06312},
  2019.

\bibitem[Ke et~al.(2019{\natexlab{a}})Ke, Xu, Zhang, Bian, and
  Liu]{ke2019deepgbm}
G.~Ke, Z.~Xu, J.~Zhang, J.~Bian, and T.-Y. Liu, ``Deepgbm: A deep learning
  framework distilled by gbdt for online prediction tasks,'' in
  \emph{Proceedings of the 25th ACM SIGKDD International Conference on
  Knowledge Discovery \& Data Mining}, 2019, pp. 384--394.

\bibitem[Ke et~al.(2019{\natexlab{b}})Ke, Zhang, Zhenhui~Xu, and Liu]{tabnn}
\BIBentryALTinterwordspacing
G.~Ke, J.~Zhang, J.~B. Zhenhui~Xu, and T.-Y. Liu, ``{TabNN}: A universal neural
  network solution for tabular data,'' 2019. [Online]. Available:
  \url{https://openreview.net/ forum?id=r1eJssCqY7}
\BIBentrySTDinterwordspacing

\bibitem[Guyon et~al.(2019)Guyon, Sun-Hosoya, Boull{\'e}, Escalante, Escalera,
  Liu, Jajetic, Ray, Saeed, Sebag, et~al.]{guyon2019analysis}
I.~Guyon, L.~Sun-Hosoya, M.~Boull{\'e}, H.~J. Escalante, S.~Escalera, Z.~Liu,
  D.~Jajetic, B.~Ray, M.~Saeed, M.~Sebag \emph{et~al.}, ``Analysis of the
  {AutoML} challenge series 2015--2018,'' in \emph{Automated Machine Learning},
  ser. Springer series on Challenges in Machine Learning.\hskip 1em plus 0.5em
  minus 0.4em\relax Springer, 2019, pp. 177--219.

\bibitem[Kingma and Ba(2015)]{kingma2014adam}
D.~Kingma and J.~Ba, ``Adam: A method for stochastic optimization,'' in
  \emph{International Conference for Learning Representations}, 2015.

\end{thebibliography}


\begin{thebibliography}{19}
\providecommand{\natexlab}[1]{#1}
\providecommand{\url}[1]{\texttt{#1}}
\expandafter\ifx\csname urlstyle\endcsname\relax
  \providecommand{\doi}[1]{doi: #1}\else
  \providecommand{\doi}{doi: \begingroup \urlstyle{rm}\Url}\fi

\bibitem[Baxter(2000)]{baxter2000model}
J.~Baxter.
\newblock A model of inductive bias learning.
\newblock \emph{Journal of artificial intelligence research}, 12:\penalty0
  149--198, 2000.

\bibitem[Bucilua et~al.(2006)Bucilua, Caruana, and
  Niculescu-Mizil]{bucilua2006compress}
C.~Bucilua, R.~Caruana, and A.~Niculescu-Mizil.
\newblock Model compression.
\newblock In \emph{Proceedings of the 12th ACM SIGKDD international conference
  on Knowledge discovery and data mining}, pages 535--541, 2006.

\bibitem[Cheng et~al.(2016)Cheng, Koc, Harmsen, Shaked, Chandra, Aradhye,
  Anderson, Corrado, Chai, Ispir, et~al.]{cheng2016wide}
H.-T. Cheng, L.~Koc, J.~Harmsen, T.~Shaked, T.~Chandra, H.~Aradhye,
  G.~Anderson, G.~Corrado, W.~Chai, M.~Ispir, et~al.
\newblock Wide \& deep learning for recommender systems.
\newblock In \emph{Proceedings of the 1st workshop on deep learning for
  recommender systems}, pages 7--10, 2016.

\bibitem[Dagan et~al.(2019)Dagan, Daskalakis, Dikkala, and
  Jayanti]{dagan2019learning}
Y.~Dagan, C.~Daskalakis, N.~Dikkala, and S.~Jayanti.
\newblock Learning from weakly dependent data under {Dobrushin's} condition.
\newblock \emph{arXiv preprint arXiv:1906.09247}, 2019.

\bibitem[Diaconis et~al.(2010)Diaconis, Khare, and
  Saloff-Coste]{diaconis2010gibbs}
P.~Diaconis, K.~Khare, and L.~Saloff-Coste.
\newblock Gibbs sampling, conjugate priors and coupling.
\newblock \emph{Sankhya A}, 72\penalty0 (1):\penalty0 136--169, 2010.

\bibitem[Erickson et~al.(2020)Erickson, Mueller, Shirkov, Zhang, Larroy, Li,
  and Smola]{erickson2020autogluon}
N.~Erickson, J.~Mueller, A.~Shirkov, H.~Zhang, P.~Larroy, M.~Li, and A.~Smola.
\newblock {AutoGluon-Tabular}: Robust and accurate {AutoML} for structured
  data.
\newblock \emph{arXiv preprint arXiv:2003.06505}, 2020.

\bibitem[Gijsbers et~al.(2019)Gijsbers, LeDell, Thomas, Poirier, Bischl, and
  Vanschoren]{gijsbers2019open}
P.~Gijsbers, E.~LeDell, J.~Thomas, S.~Poirier, B.~Bischl, and J.~Vanschoren.
\newblock An open source {AutoML} benchmark.
\newblock In \emph{ICML Workshop on Automated Machine Learning}, 2019.

\bibitem[Guo and Berkhahn(2016)]{guo2016entity}
C.~Guo and F.~Berkhahn.
\newblock Entity embeddings of categorical variables.
\newblock \emph{arXiv preprint arXiv:1604.06737}, 2016.

\bibitem[Haussler(1990)]{haussler1990probably}
D.~Haussler.
\newblock \emph{Probably approximately correct learning}.
\newblock University of California, Santa Cruz, Computer Research Laboratory,
  1990.

\bibitem[Hinton et~al.(2015)Hinton, Vinyals, and Dean]{hinton2015distilling}
G.~Hinton, O.~Vinyals, and J.~Dean.
\newblock Distilling the knowledge in a neural network.
\newblock \emph{NIPS Deep Learning and Representation Learning Workshop}, 2015.

\bibitem[Hoogeboom et~al.(2020)Hoogeboom, Cohen, and
  Tomczak]{hoogeboom2020learning}
E.~Hoogeboom, T.~S. Cohen, and J.~M. Tomczak.
\newblock Learning discrete distributions by dequantization.
\newblock \emph{arXiv preprint arXiv:2001.11235}, 2020.

\bibitem[Howard and Gugger(2020)]{fastai}
J.~Howard and S.~Gugger.
\newblock fastai: A layered api for deep learning.
\newblock \emph{arXiv preprint arXiv:2002.04688}, 2020.

\bibitem[Li et~al.(2017)Li, Chang, Cheng, Yang, and P{\'o}czos]{li2017mmd}
C.-L. Li, W.-C. Chang, Y.~Cheng, Y.~Yang, and B.~P{\'o}czos.
\newblock Mmd gan: Towards deeper understanding of moment matching network.
\newblock In \emph{Advances in Neural Information Processing Systems}, pages
  2203--2213, 2017.

\bibitem[Ma et~al.(2019)Ma, Kong, Zhang, and Hovy]{ma2019macow}
X.~Ma, X.~Kong, S.~Zhang, and E.~Hovy.
\newblock Macow: Masked convolutional generative flow.
\newblock In \emph{Advances in Neural Information Processing Systems}, pages
  5891--5900, 2019.

\bibitem[Robert and Casella(2013)]{robert2013monte}
C.~Robert and G.~Casella.
\newblock \emph{Monte Carlo statistical methods}.
\newblock Springer Science \& Business Media, 2013.

\bibitem[Theis et~al.(2016)Theis, van~den Oord, and Bethge]{theis2016note}
L.~Theis, A.~van~den Oord, and M.~Bethge.
\newblock A note on the evaluation of generative models.
\newblock In \emph{International Conference on Learning Representations}, 2016.

\bibitem[Vapnik and Chervonenkis(2015)]{vapnik2015uniform}
V.~N. Vapnik and A.~Y. Chervonenkis.
\newblock On the uniform convergence of relative frequencies of events to their
  probabilities.
\newblock In \emph{Measures of complexity}, pages 11--30. Springer, 2015.

\bibitem[Vaswani et~al.(2017)Vaswani, Shazeer, Parmar, Uszkoreit, Jones, Gomez,
  Kaiser, and Polosukhin]{vaswani2017attention}
A.~Vaswani, N.~Shazeer, N.~Parmar, J.~Uszkoreit, L.~Jones, A.~N. Gomez,
  L.~Kaiser, and I.~Polosukhin.
\newblock Attention is all you need.
\newblock In \emph{Advances in Neural Information Processing Systems}, 2017.

\bibitem[Wang et~al.(2014)Wang, Wu, et~al.]{wang2014convergence}
N.-Y. Wang, L.~Wu, et~al.
\newblock Convergence rate and concentration inequalities for gibbs sampling in
  high dimension.
\newblock \emph{Bernoulli}, 20\penalty0 (4):\penalty0 1698--1716, 2014.

\end{thebibliography}

}
%\bibliographystyle{abbrvnat}

%%%%%%%%%%%%%%%%%%%%%%%%%%%%%%%%%%%%%%%%%%%%%%%%
%%%  Supplement  %%%%
%%%%%%%%%%%%%%%%%%%%%%%%%%%%%%%%%%%%%%%%%%%%%%%%

% Use \citesi{}, \citetsi{}, \citepsi{} for supplementary citations.

% run: bibtex si.aux to generate supplementary bibliography

\clearpage \newpage
\beginsupplement

\appendix
\setcounter{page}{1}

\begin{center}
{\Large \bf Appendix: \ 
Fast, Accurate, and Simple Models for Tabular Data \\
\hspace*{-30.5mm} via Augmented Distillation 
}
% \bottomtitlebar
\end{center}

\section{Methods Details}

We not only adopt the  AutoGluon\footnote{\url{https://github.com/awslabs/autogluon/}} predictor as our teacher for distillation, but our experiments also use the AutoGluon implementation of each individual model type (NN, RF, LightGBM, CatBoost) for our student/BASE predictors\footnote{Although we selected AutoGluon as the AutoML tool for this paper's experiments, we emphasize that none of our distillation methodology is specific to AutoGluon teachers/students.}. 
Here we consider the same data preprocessing and hyperparameters as AutoGluon uses by default, which have been demonstrated to be highly performant~\citesi{erickson2020autogluon}. 

Unlike the RF/LightGBM/CatBoost models which are implemented in popular third party packages, the NN model is implemented directly in AutoGluon, and offers numerous advantages for tabular data over standard feedforward architectures \citesi{erickson2020autogluon}. 
This network uses a separate embedding layer for each categorical feature which helps the network separately learn about each of these variables before their representations are blended together by fully-connected layers \citesi{guo2016entity,fastai}. The network employs skip connections for improved gradient flow, with both shallow and deep paths connecting the input to the output \citesi{cheng2016wide}.

Note that all of our student classifiers produce valid predicted probabilities: our neural network student employs a sigmoid output layer to constrain its outputs to $[0,1]$ in binary classification, and the random forest multiclass students never output negative values (these models do not extrapolate) so we can simply re-normalize their output vectors to have unit-sum.

% \subsection{Neural Network Architecture}

\subsection{Architecture of our Pseudolikelihood Self-Attention Model}

The input layer of our self-attention network applies a linear embedding operation followed by positional encoding.
Each internal layer of the network is a Transformer block, which includes two sub-blocks: a multi-head self-attention mechanism and a position-wise fully connected feedforward block~\cite{vaswani2017attention}. Each of these sub-blocks is wrapped with layer normalization and a residual connection. Here different positions correspond to different features (columns of the table). 
The output layer of this network produces a mixture of multivariate Gaussians with diagonal covariance, where the final position-wise feedforward block outputs for each feature $i$ both the mean/variance of each Gaussian component ($\mu_k,\sigma_k$) as well as the mixing components ($\lambda_k$). In order to make sure that all input features are on a similar scale, all features are rescaled to mean-zero unit-variance before being fed into our network (and we apply the inverse transform after Gibbs sampling).

Positional encoding is essential for the model to know which value was taken by which feature. For example, without positional encoding:   $x^{(1)} = 1, x^{(2)} = 0$ v.s.\ $x^{(1)} = 0, x^{(2)} = 1$ would lead to similar self-attention input for the third feature $x^{(3)}$ without positional encoding. Thus the representations of our model would suffer, as would its estimated conditional distributions. Here we employ the same sin/cos positional encodings used by \citeauthor{vaswani2017attention}  \citesi{vaswani2017attention}, treating the table column-index of each feature analogously to word positions in a sentence.

Tabular data can contain both numerical and categorical features. In order to have a simple, unified model that can deal with both feature types, we represent categorical features numerically using dequantization
~\cite{UriaRNADE}. This involves adding uniform noise an the ordered integer encoding of the categories to make these features look numerical to our network. The noise can be inverted via rounding to ensure that discrete categories are produced by our Gibbs sampler  (i.e.\ re-quantization). 
Dequantization has been successfully employed in a number of deep architectures that otherwise operate on continuous data  \citesi{hoogeboom2020learning,theis2016note, ma2019macow}, and allows us to avoid having to employ  heterogeneous output layers and unwieldy one-hot enodings.

\cref{tab:hp} shows our network's hyper-parameters that are used for the experiments in this paper. It is worth noting that we did not conduct any hyper-parameter search to find the best-performing architectures and models. Instead, we simply utilize two different networks: \emph{Small} and \emph{Large}. The \emph{Small} network is used whenever the training dataset has less than $15000$ examples and we otherwise use the \emph{Large} network. Their only differences are in batch sizes and the width of their hidden layers, all other details such as training procedure, regularization, evaluation protocol, etc. are the same. We  utilize two different models in order to avoid overfitting small datasets, and the \emph{Small} network can also be more efficiently  trained. 
% It is likely to further improve the performance of our method with a hyper-parameter search. 
We use Adam to optimize the parameters of our network~\cite{kingma2014adam}.

\begin{table*}[!b]
\begin{center}
\begin{tabular}{p{5cm} rr}
& \emph{Small} & \emph{Large} \\
\toprule
\toprule
Gaussian mixture components & 100 & 100 \\
Number of layers            & 4 & 4 \\
Multi-head attention heads   & 8 & 8 \\
Hidden unit size              & 32 & 128 \\
Mini-batch size             & 16 & 256 \\
Dropout                     & 0.1 & 0.1 \\
Learning rate               & 3E-4 & 3E-4\\
Weight decay                & 1E-6 & 1E-6 \\
Gradient clipping norm      & 5 & 5 \\
\bottomrule
\end{tabular}
\end{center}
\caption{Hyper-parameters of our self-attention models.}
\label{tab:hp}
\vspace*{1em}
\end{table*}

\begin{table*}[!b]
\centering
\begin{tabular}{llllc}
\toprule
\textbf{Dataset} & \textbf{Type} & \textbf{Sample Size} & \textbf{\# Columns} & \textbf{\# Classes} \\
\midrule
        amazon &        binary &                32769 &                     9 &                    -  \\
    australian &        binary &                  690 &                    14 &                    -  \\
     miniboone &        binary &               130064 &                    50 &                    -  \\
         adult &        binary &                48842 &                    14 &                    -  \\
         blood &        binary &                  748 &                     4 &                    -  \\
      credit-g &        binary &                 1000 &                    20 &                    -  \\
         higgs &        binary &                98050 &                    28 &                    -  \\
       jasmine &        binary &                 2984 &                   144 &                    -  \\
         nomao &        binary &                34465 &                   118 &                    -  \\
   numerai28.6 &        binary &                96320 &                    21 &                    -  \\
       phoneme &        binary &                 5404 &                     5 &                    -  \\
       sylvine &        binary &                 5124 &                    20 &                    -  \\
     covertype &    multiclass &               581012 &                    54 &                     7 \\
        helena &    multiclass &                65196 &                    27 &                   100 \\
        jannis &    multiclass &                83733 &                    54 &                     4 \\
       volkert &    multiclass &                58310 &                   180 &                    10 \\
     connect-4 &    multiclass &                67557 &                    42 &                     3 \\
  jungle-chess &    multiclass &                44819 &                     6 &                     3 \\
 mfeat-factors &    multiclass &                 2000 &                   216 &                    10 \\
       segment &    multiclass &                 2310 &                    19 &                     7 \\
       vehicle &    multiclass &                  846 &                    18 &                     4 \\
        boston &    regression &                  506 &                    13 &                    -  \\
      concrete &    regression &                 1030 &                     8 &                    -  \\
        energy &    regression &                  768 &                     8 &                    -  \\
        kin8nm &    regression &                 8192 &                     8 &                    -  \\
         naval &    regression &                11934 &                    16 &                    -  \\
         power &    regression &                 9568 &                     4 &                    -  \\
       protein &    regression &                45730 &                     9 &                    -  \\
          wine &    regression &                 1599 &                    11 &                    -  \\
         yacht &    regression &                  308 &                     6 &                    -  \\
\bottomrule
\end{tabular}

\caption{Summary of 30 datasets considered in this work, listing the: type of prediction problem, size of the data table, and number of classes for multiclass classification problems. The regression data (along with provided train/test splits) were downloaded from: \url{https://github.com/yaringal/DropoutUncertaintyExps}. The classification data  (with provided train/test splits) were downloaded from:  \url{https://www.openml.org/s/218}. We initially considered  additional classification datasets from {\citet{gijsbers2019open}}, but decided to not to include those for which: it was trivial to get near 100\% accuracy for many model types (so a teacher is unnecessary), the data are dominated by missing values, the original data are extremely high-dimensional ($d > 1000$), or 
the original data did not come from a table (e.g.\ Fashion-MNIST).
}
\label{tab:datasets}
\vspace*{1em}
\end{table*}

\clearpage
\section{Experiment Details}
\label{sec:experimentdetails}

We implemented knowledge distillation (KNOW) with classification targets modified as suggested in \citetsi{hinton2015distilling}.
As suggested by \citetsi{bucilua2006compress}, the distance metric in MUNGE is taken to be the Euclidean distance between (rescaled) numerical features and the Hamming distance between categorical features. Over all datasets, we performed a grid search over MUNGE's user-specified parameters: the feature-resampling probability $p$ and local variance parameter $s$, in order to maximize validation accuracy of the student over $p \in \{ 0.1, 0.25, 0.5, 0.75 \}, s \in \{ 0.5, 1.0, 5.0 \}$. For the conditional tabular GAN, we used the original implementation available at: \url{https://github.com/sdv-dev/CTGAN}.

On each dataset, we trained AutoGluon for up to 4 hours, and specified the same time-limit for H2O-AutoML and AutoSklearn. When running H2O and AutoSklearn on the 30 datasets, each AutoML tool failed to produce predictions on 2 datasets, and we simply recorded the accuracy/latency achieved by the other tool in this case (such failures are common in AutoML benchmarking, c.f.\ \citesi{gijsbers2019open,erickson2020autogluon}). Each AutoML tool was run with all default arguments, except for AutoGluon: we additionally set the argument  \texttt{auto\_stack = True} which instructs the system to maximize accuracy at all costs via extensive stack ensembling. We used the same type of AWS EC2 instance (m5.2xlarge) for each predictor to ensure fair comparison of inference times (each tool was run on separate EC2 instance with no other running processes).

For evaluating our Gibbs samples, we computed the Maximum Mean Discrepancy with the mixture-kernel \citesi{li2017mmd}, with bandwidths = $[1, 2, 4, 8, 16]$.
Our procedure to measure sample fidelity involved the following steps: 
First we trained our \transformer{} to maximize pseudolikelihood over data in the training fold.  
Next we applied Gibbs sampling to generate synthetic samples (initializing the Markov chains at the training data as previously described). 
Subsequently we assembled a balanced dataset of real (held-out) data from our validation fold which received label $y=1$ and fake data comprised of Gibbs samples which received label $y=0$. A random forest was trained on this dataset, and then its accuracy evaluated on another balanced dataset comprised of real data from our test fold (again with label $y=1$) and fake data comprised of a different set of Gibbs samples (labeled with $y=0$).
The resulting `sample fidelity' was defined as the distance between this RF accuracy and 0.5.

\iffalse
% TODO: outdated text, remove
The AutoML Benchmark is intended to be a heterogeneous collection of data tables representative of the types of data encountered in practical classification applications. While its curation is certainly better in this regard than open repositories like OpenML or the UCI Machine Learning Repository, we find additional curation of these datasets leads to more insightful comparison of ML algorithms.  In our analyses, we omit the following types of  datasets:

1) those for which it is too trivial to achieve $\ge 99.9\%$ accuracy, which includes: car, Shuttle, kr-vs-kp. Many of these are toy datasets that do not stem from real applications, and the base models trained directly on the data cannot be improved further given their near perfect predictions.

2) those which are not actually tabular data (but are actually say images), which include: Fashion-MNIST.

3) those for which our model ensemble is outperformed by individual base models on the test data (despite exhibiting superior validation performance), possibly due to some overfitting during ensemble construction, which include: Dionis. One cannot expect additional gains from distilling a model that was already more accurate than the teacher after being trained directly on the data.
\fi

%%%% Extra results %%%%
% \FloatBarrier
\clearpage
\section{Additional Results}
\FloatBarrier

\begin{figure}[h!] \centering
\begin{center}
\textbf{(A) Regression} \\[1em]
\includegraphics[width=0.5\textwidth]{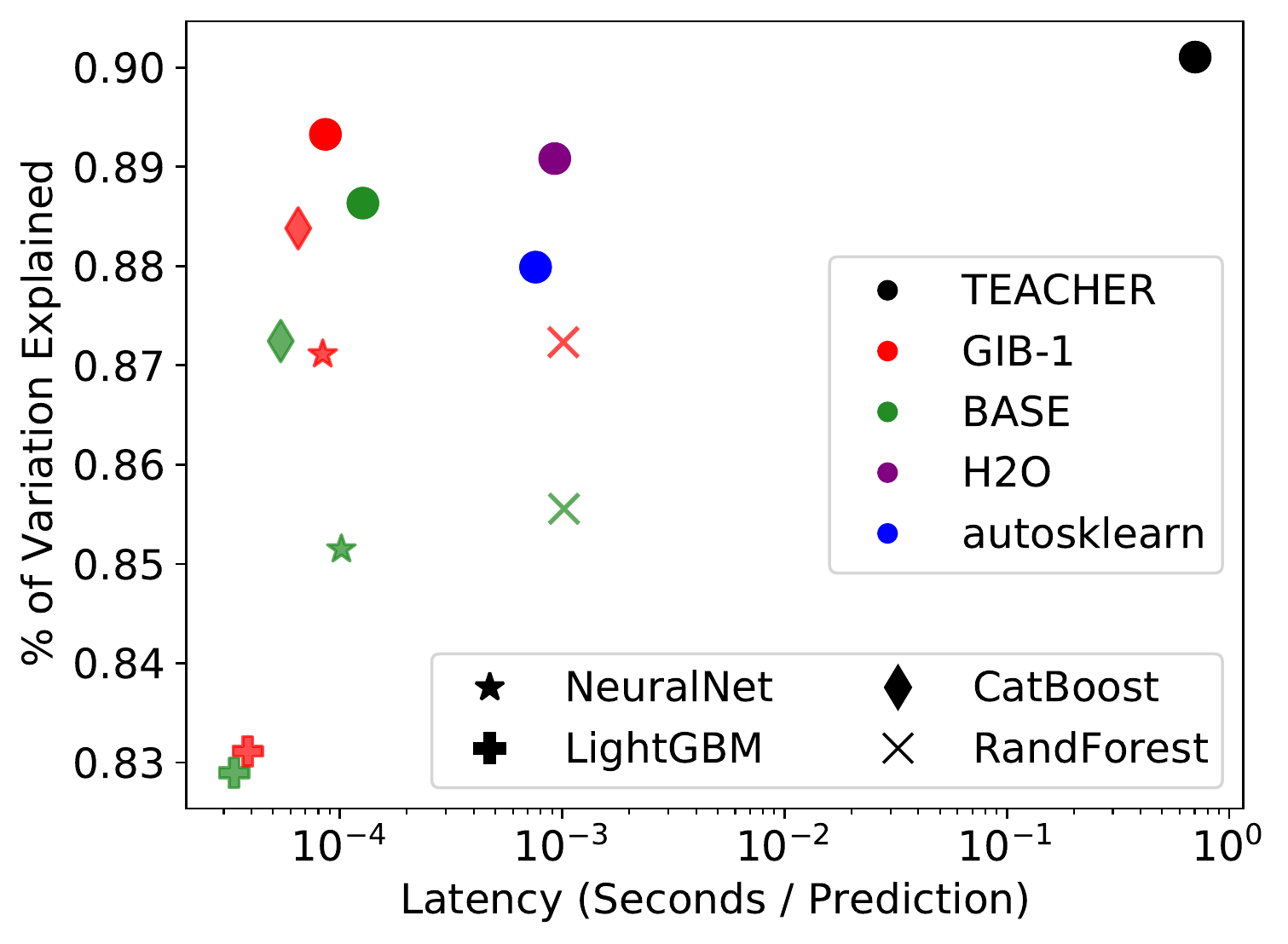} \\[1em]
\end{center}
\begin{tabular}{cc}
\textbf{(B) Binary Classification} & \textbf{\hspace*{3mm} (C) Multiclass Classification}
\\
\includegraphics[width=0.48\textwidth]{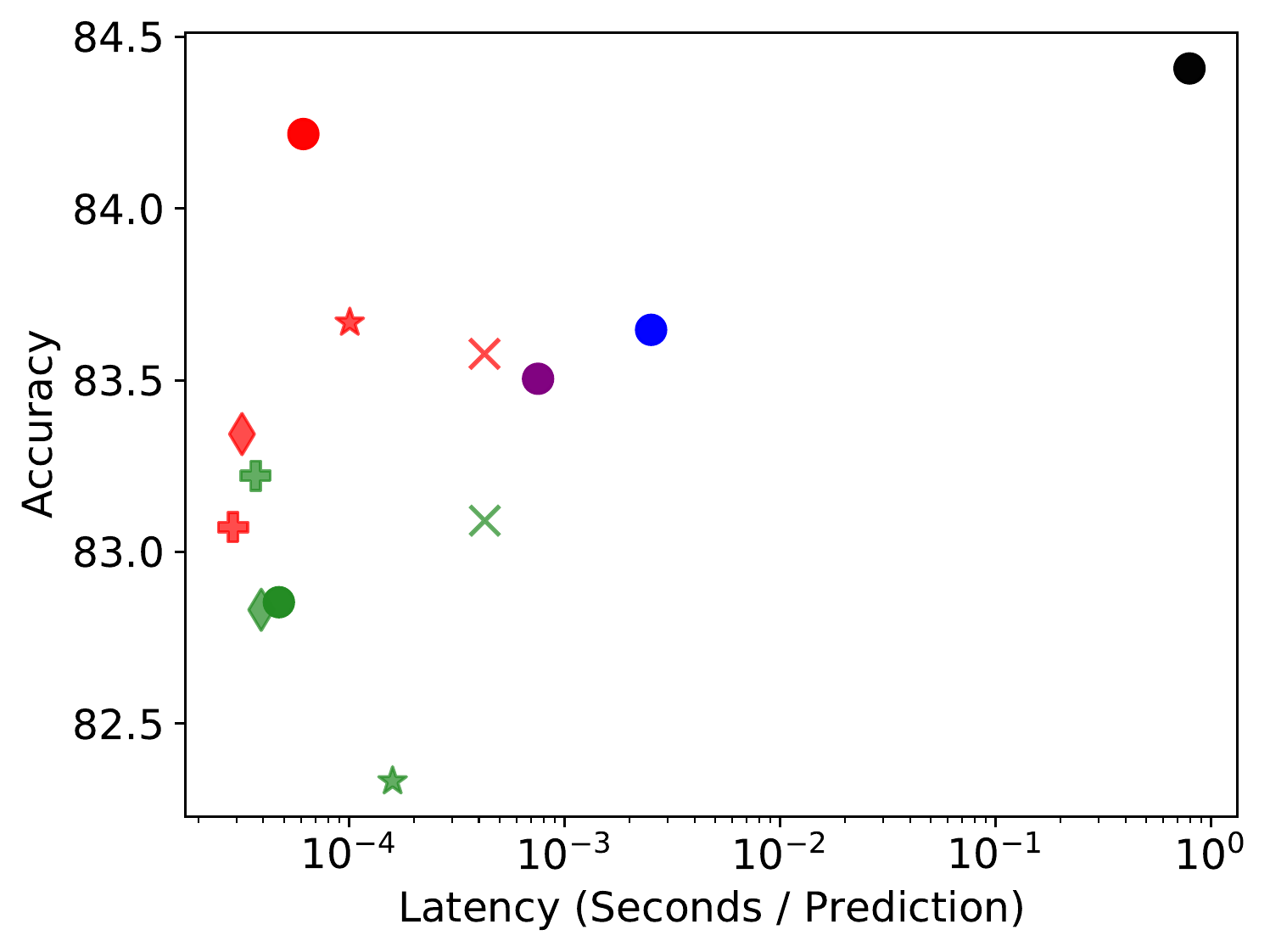}
&
\includegraphics[width=0.47\textwidth]{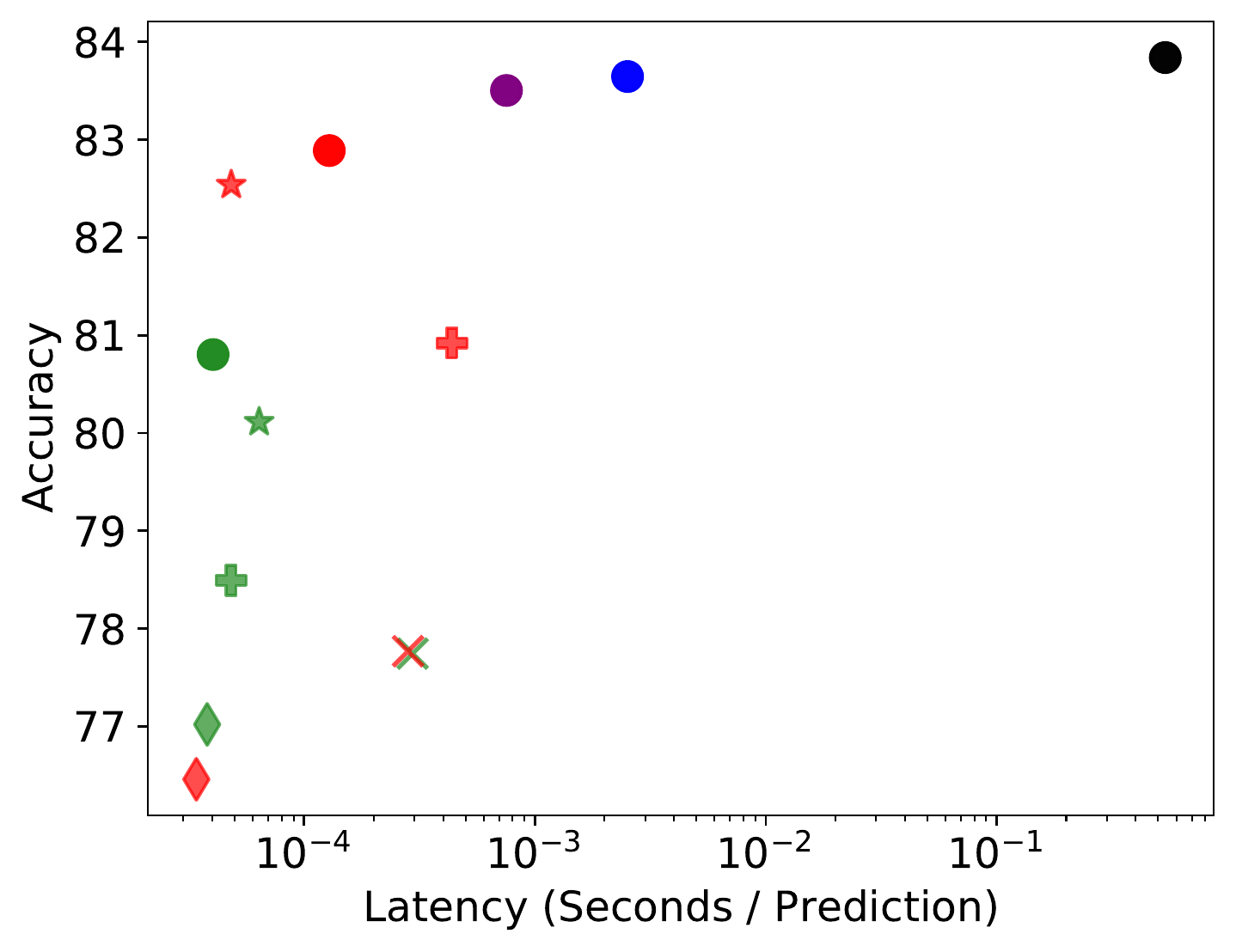}
\end{tabular}
\caption{Test accuracy vs latency of individual models and AutoML ensembles, averaged over the: \textbf{(A)} regression datasets, \textbf{(B)} binary classification datasets, \textbf{(C)} multiclass classification datasets. The GIB-1 and BASE dots show performance of \emph{Selected} model (out of the 4 types) on each dataset. Note that for binary classification: the \emph{Selected} BASE models are actually worse than individual RF/LightGBM models, presumably due to overfitting of the validation set via the early-stopping criterion in NN/CatBoost. This issue appears to be mitigated by distillation with augmented data. In multiclass classification, the distilled LightGBM models exhibit worse latency than their BASE counterparts because distillation uses additional data and soft (probabilistic) labels as targets, such that the underlying function to learn becomes more complex. Thus, the depth of its trees grows since LightGBM does not limit it by default. The BASE/distilled latency could easily be matched by restrictively setting the LightGBM depth/leaf-size hyperparameters to ensure equal-sized trees in these two variants.
}
\label{fig:automlregressmulti}
\end{figure}

% Raw accuracy table of SEL models
\begin{table*}[h!]
\centering
\caption{Raw test accuracy (or percent of variation explained $= R^2 \cdot 100$ for regression) under various training/distillation strategies of the \emph{Selected} best individual model (across all 4 model types) chosen based on validation performance. The final column shows the performance of the ensemble-predictor produced by AutoGluon (used as teacher in distillation). Datasets are colored by task: regression (black), binary classification (blue), multiclass classification (red).
}
\label{tab:selaccs}
\vspace*{1em}
\addtolength{\tabcolsep}{-1.1pt}
\begin{footnotesize}
\begin{tabular}{lccccccccc}
\toprule
                 \textbf{Dataset}  &  \textbf{BASE} &   \textbf{KNOW} &  \textbf{MUNGE} &  \textbf{HUNGE} &    \textbf{GAN} &  \textbf{GIB-1} &  \textbf{GIB-5} & \textbf{GIB-10} & \hspace*{-2mm} \textbf{TEACHER} \hspace*{-2mm} \\
\midrule
                         boston &           91.84 &           90.11 &           90.25 &           91.54 &           92.02 &           92.38 &  \textbf{93.21} &           92.62 &            92.09 \\
                       concrete &           92.20 &           92.66 &           92.07 &           92.31 &           92.33 &           92.39 &  \textbf{92.83} &           92.56 &            92.82 \\
                         energy &           99.86 &           99.85 &           99.91 &           99.92 &           99.87 &  \textbf{99.93} &           99.92 &           99.92 &            99.93 \\
                         kin8nm &           93.36 &           93.58 &           94.10 &           93.82 &           94.08 &           93.96 &  \textbf{94.14} &           94.10 &            93.99 \\
                          naval &           99.74 &           99.75 &  \textbf{99.81} &           99.78 &           99.49 &           99.70 &           99.68 &           99.71 &            99.97 \\
                          power &           96.62 &  \textbf{96.97} &           96.60 &           96.86 &           96.07 &           96.61 &           96.62 &           96.51 &            97.15 \\
                        protein &           68.34 &           67.37 &           69.95 &           68.14 &           67.64 &  \textbf{69.96} &           69.07 &           68.01 &            74.33 \\
                           wine &           56.44 &           56.53 &           57.38 &           58.61 &           56.42 &  \textbf{59.27} &           57.80 &           58.49 &            60.74 \\
                          yacht &           99.24 &           99.55 &           99.88 &           99.92 &           99.93 &  \textbf{99.94} &           99.94 &           99.90 &            99.87 \\
       \textcolor{blue}{amazon} &           94.84 &           94.81 &           94.72 &           94.87 &           94.69 &  \textbf{94.90} &           94.81 &           94.81 &            94.96 \\
   \textcolor{blue}{australian} &           86.95 &  \textbf{88.40} &           85.50 &           85.50 &           85.50 &  \textbf{88.40} &           86.95 &           85.50 &            86.95 \\
    \textcolor{blue}{miniboone} &           94.50 &           94.71 &           94.77 &           94.40 &           94.44 &  \textbf{94.86} &           94.44 &           94.64 &            94.88 \\
        \textcolor{blue}{adult} &           87.06 &           87.26 &           87.36 &  \textbf{87.49} &           86.81 &           86.36 &           86.67 &           86.73 &            87.59 \\
        \textcolor{blue}{blood} &           73.33 &           77.33 &           77.33 &           77.33 &            76.0 &            76.0 &  \textbf{78.66} &           77.33 &             76.0 \\
     \textcolor{blue}{credit-g} &            71.0 &            78.0 &            78.0 &            76.0 &            75.0 &   \textbf{80.0} &   \textbf{80.0} &            77.0 &             79.0 \\
        \textcolor{blue}{higgs} &           72.14 &           73.14 &           73.53 &           72.83 &           72.48 &  \textbf{73.89} &           73.36 &           73.18 &            73.83 \\
      \textcolor{blue}{jasmine} &  \textbf{82.94} &           81.93 &           80.26 &           81.93 &           81.93 &           82.27 &           81.27 &           81.93 &            82.60 \\
        \textcolor{blue}{nomao} &  \textbf{97.30} &           96.98 &           96.72 &           97.15 &           96.98 &           96.77 &           96.83 &           96.86 &            98.20 \\
  \textcolor{blue}{numerai28.6} &           51.12 &           50.36 &           51.78 &           50.78 &           51.23 &  \textbf{52.05} &           51.11 &           51.59 &            51.10 \\
      \textcolor{blue}{phoneme} &           89.46 &           90.38 &           90.57 &           90.20 &           90.57 &  \textbf{91.49} &           90.38 &           90.75 &            92.42 \\
      \textcolor{blue}{sylvine} &           93.56 &           93.37 &           94.15 &           94.34 &           92.39 &           93.56 &           93.95 &  \textbf{94.54} &            95.32 \\
     \textcolor{red}{covertype} &           95.90 &           96.99 &  \textbf{97.00} &           92.84 &           96.39 &           96.19 &           96.48 &           96.06 &            97.66 \\
        \textcolor{red}{helena} &           38.29 &  \textbf{40.70} &           40.26 &           39.44 &           39.43 &           40.50 &           39.84 &           40.50 &            40.75 \\
        \textcolor{red}{jannis} &           70.69 &           72.23 &           72.13 &           70.69 &           71.75 &  \textbf{72.43} &           72.28 &           71.91 &            73.07 \\
       \textcolor{red}{volkert} &           69.62 &           71.34 &  \textbf{72.18} &           69.28 &           70.60 &           71.42 &           70.70 &           70.45 &            74.46 \\
     \textcolor{red}{connect-4} &           84.87 &           86.10 &           86.27 &           84.35 &           85.90 &           86.19 &  \textbf{86.44} &           86.38 &            86.04 \\
  \textcolor{red}{jungle-chess} &           87.59 &           91.78 &           92.92 &           89.53 &  \textbf{96.07} &           93.37 &           94.42 &           93.81 &            99.55 \\
 \textcolor{red}{mfeat-factors} &            98.0 &            97.0 &            97.5 &            97.5 &            98.0 &   \textbf{98.5} &   \textbf{98.5} &   \textbf{98.5} &             98.0 \\
       \textcolor{red}{segment} &           98.70 &           98.70 &           98.70 &           98.70 &           98.70 &  \textbf{99.13} &  \textbf{99.13} &  \textbf{99.13} &   99.13 \\
       \textcolor{red}{vehicle} &           83.52 &           77.64 &  \textbf{88.23} &           87.05 &           83.52 &  \textbf{88.23} &           87.05 &           87.05 &            85.88 \\
\bottomrule
\end{tabular}

\end{footnotesize}
\addtolength{\tabcolsep}{1.1pt}
% \end{center}
\vskip -0.1in
\end{table*}

\begin{figure}[h!] \centering
\begin{tabular}{cc}
\textbf{ \hspace*{5mm} (A) Augmentation w Real Data} & \textbf{\hspace*{5mm} (B) Augmentation w GIB-1}
\\
\includegraphics[width=0.49\textwidth]{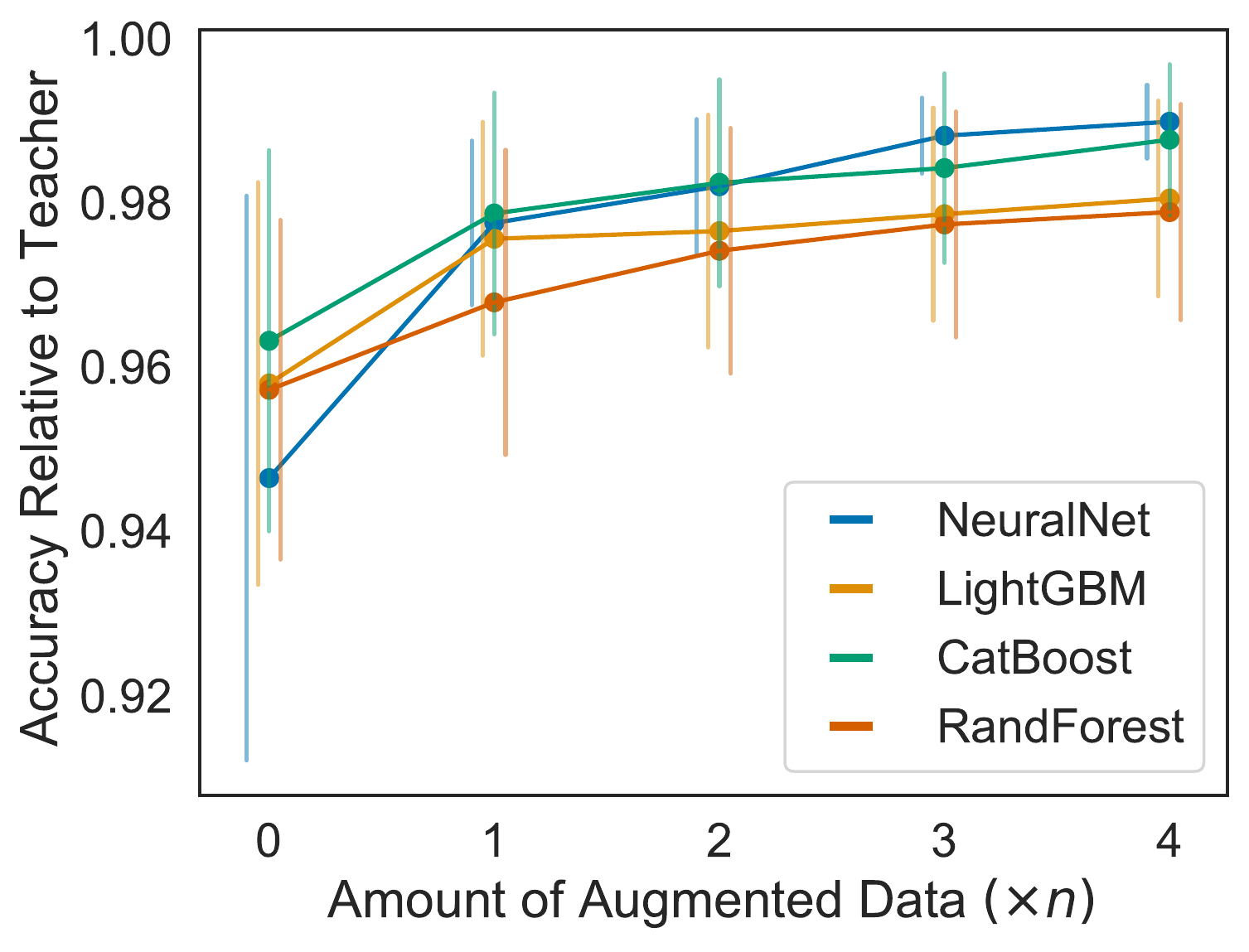}
&
\includegraphics[width=0.49\textwidth]{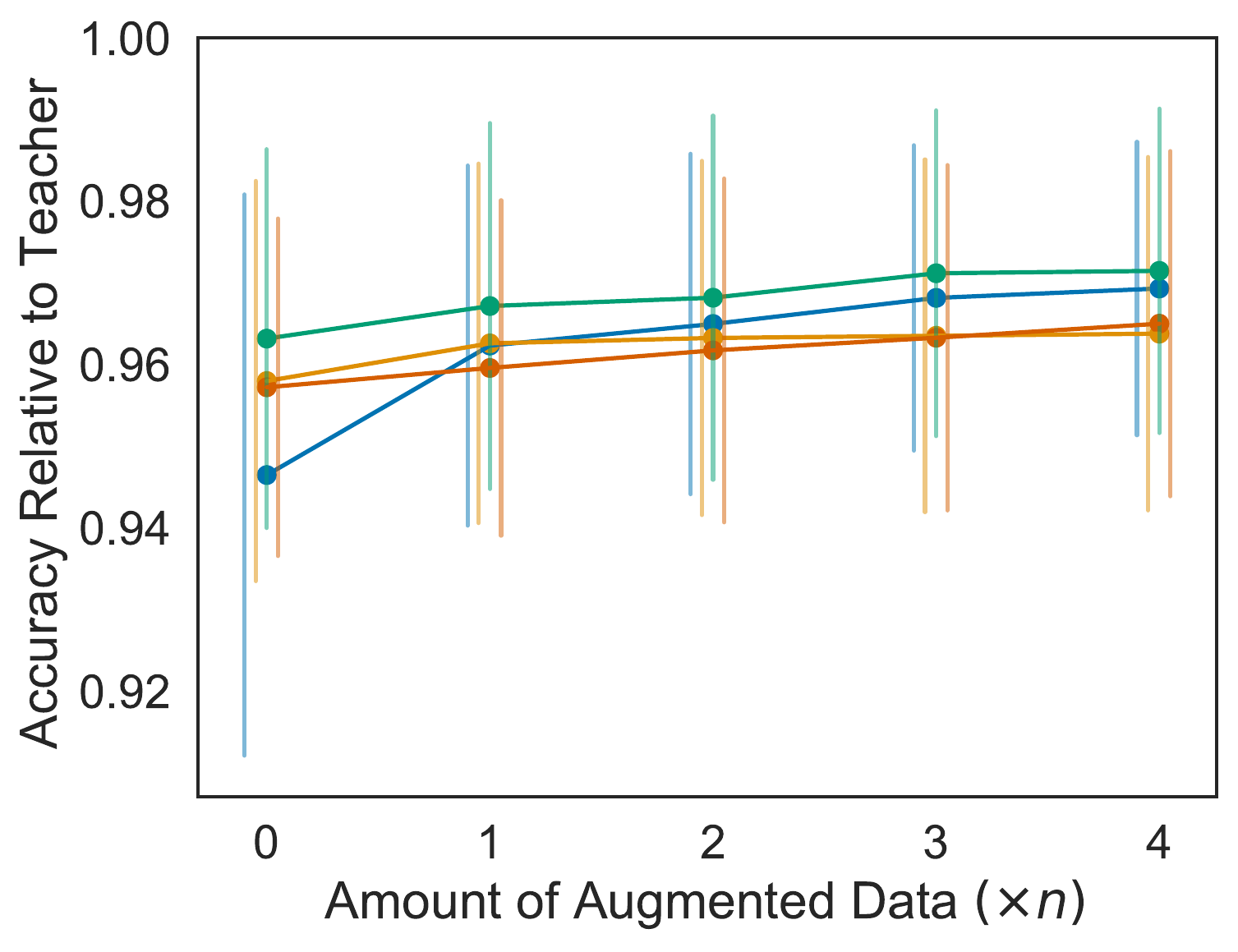}
\end{tabular}

\caption{Distillation performance when augmented data are:  \textbf{(A)} additional real data points from the true underlying  distribution, \textbf{(B)} synthetic examples obtained from 1 round of our Gibbs sampling procedure. Here we report average normalized test accuracy ($R^2$) over the 3 largest regression datasets, with corresponding standard errors indicated by vertical lines  (our normalization rescales $R^2$ by the teacher's $R^2$ on each dataset). To obtain additional real data points for augmentation, we did the following: only 20\% of the original training set was adopted as the training data (accuracies obtained from this training data shown at 0 on x-axis). The rest of the  80\% held-out data was treated as unlabeled and used as augmented data for distillation (in increasing multiples of the training sample-size $n$), following the same distillation procedure described in the main text. The GIB-1 results are obtained by applying our FAST-DAD distillation procedure with only the 20\% training data (the 80\% held-out data are entirely ignored, so our self-attention pseudolikelihood model is fit to relatively little data). The AutoGluon teacher is also only fit to the same 20\% of the training data.  
Panel \textbf{(A)}  empirically validates Lemma \ref{lem:surrogate}, showing that distillation becomes much more powerful with additional unlabeled data from the true feature distribution. Distillation gains produced by augmenting with Gibbs samples do not match the performance of augmenting with real data, suggesting superior generative models may further reduce this gap. % (if they can be accurately learned from limited data). 
}
\label{fig:unlabeledperf}
\end{figure}

\clearpage
\section{Proof of \cref{thm:refinement}}
\label{s:proof_refinement}

Here we discuss our refinement of Lemma \ref{lem:surrogate} that formally describes how the number of steps of Gibbs sampling affects the distillation of the student. Lemma \ref{lem:surrogate} suggests that if we learn a probability distribution $q$ using the data $X_n$, we might be able to reduce the variance term in the VC-bound at the cost of a bias. We now characterize the situation when the Gibbs sampler with a steady-state distribution $q$ is initialized at samples from $p$, namely the original training dataset $X_n$, and is run for $k$ steps. Intuitively, if $k$ is large, the sampler provides data $X_m'$ that is diverse from $X_n$ which leads to stronger variance reduction. However it is also true that the samples $X_m'$ are not drawn from $p$ and therefore the teacher $f$ suffers a covariate shift on these samples which leads to poor fitting of the student $g$. This suggests  there should be a sweet spot: the number of Gibbs sampling steps $k$ should lead to variance reduction but should not be so large as to cause a large covariate-shift/bias. We capture this phenomenon in the following theorem. For simplicity, we only consider the special case where $m = n$. Our proof can be generalized to $m \neq n$ but the details of the underlying symmetrization argument are more intricate (see comments in the proof). We stick to this special case to elucidate the main point. 
The full theorem statement is repeated here for completeness.

\setcounter{theorem}{1}
\begin{theorem}[Refinement of Lemma 1]
Under the assumptions of Lemma 1, suppose that the student $g^*$ is chosen to minimize $D_{\trm{emp}}(f, g, X_n \cup X_n')$ where $X_n'$ are $n$ samples drawn after running the Gibbs sampler initialized at samples from $X_n$ for $k$-steps. Then there exist constants $V, c$ and $\delta > 0$ such that with probability at least $1-\delta$ we have
\begin{align}
D(f, g^*, p) \le 
    D_{\trm{emp}}(f, g^*, X_n \cup X_n') + \sqrt{\f{4 V(c + \Delta_k) - \log \delta}{n}} + \Delta_k.
\end{align}
The quantity $\Delta_k = \norm{T^k_q p - p}_{\trm{TV}}$ is the total-variation distance between the true data distribution $p$ and the distribution of the sampler's iterates after $k$ steps, denoted by $T^k_q p$. The steady-state distribution of the Gibbs sampler is denoted by $q$.
\end{theorem}
\begin{proof}
Let $q$ be the steady-state distribution of the Gibbs sampler with a linear operator $T_q$ that denotes the one-step transition kernel. Under  general conditions~\citesi{robert2013monte}, the distribution of the iterates of the sampler converges to this steady-state distribution as $k \to \infty$, i.e.,
\[
    \lim_{k \to \infty} T^k_q \nu = q,
\]
from any initial distribution $\nu$. Explicit rates are available for this convergence~\citesi{wang2014convergence}: there exist constants $\lambda \in (0,1)$ and $c(q)$ such that
\begin{align}
    \label{eq:gibbs_convergence}
    \norm{T^k_q p  - q}_{\trm{TV}} \leq c(q) \lambda^k.
\end{align}
where $\norm{\nu - \mu}_{\trm{TV}} := 2 \sup \cbr{\abs{\nu(A) - \mu(A)}: A \in \mathcal{B}(\Xcal)}$ denotes the total-variation norm; the set $\mathcal{B}(\Xcal)$ is the Borel $\sigma$-algebra of the domain $\Xcal$.
These rates are sharp for some parametric models~\citesi{diaconis2010gibbs}. We use the following shorthand to denote, $T^k_q p$, the density obtained after applying the one-step transition kernel $k$ times.
\[
    q^k := T^k_q p.
\]

Suppose that the Gibbs sampler initialized at $p$ runs for $k$ steps and we then sample a dataset $X_n'$ of $n$ samples from the resultant distribution $T^k_q p$:
\[
    X_n' = \cbr{x'_i \sim T_q^k p}_{i=1,\ldots,n}.
\]
The samples in $X_n'$ are correlated with those already in $X_n$. The student is fit to this dataset $X_n \cup X_n'$ where the samples are not independent (we don't have $X_n \ni x \indep x' \in X_n'$) or identically distributed ($x \sim p$ and $x' \sim T^k_q p$). Characterizing generalization performance is difficult for this scenario and requires strong assumptions, c.f.\ \citesi{dagan2019learning}, but we can we make the following helpful simplification.

\begin{assumption}
The number of Gibbs steps $k$ is large enough for the samples in $X_n$ and $X_n'$ to be statistically independent.
\end{assumption}
Note that this does not imply that the samples are identically distributed, they still come from distributions $p$ and $T^k_q p$ respectively. Since $k$ is the product of the number of rounds of Gibbs sampling and the dimensionality of the data ($d$), achieving approximate independence does not necessarily require a large number of Gibbs rounds.
 
We now employ a bound by Jonathan Baxter~\citesi{baxter2000model} that studies the generalization performance of a model $g$ when it sees data from a mixture of two different, possibly correlated, distributions, $p$ and $q^k$. This is a uniform-convergence bound and follows via a two-step symmetrization argument where the second step involves separate permutations of the samples in datasets $X_n$ and $X_n'$. The same technique as that of~\citesi{baxter2000model} also works if we draw more data $X_m'$ from the new distribution than the original dataset $X_n$, i.e., if $m \geq n$. However the details are intricate and we stick to this special $m=n$ case to elucidate the main point.

For all functions $g \in \Gcal$, in particular for $g^* = \argmin_g D_{\trm{emp}}(f, g, X_n \cup X_n')$, the following holds with probability at least $1-\delta$: 
\begin{align}
    \label{eq:baxter}
    D \rbr{f, g, \f{p+q^k}{2}} &\leq D_{\trm{emp}}(f, g, X_n \cup X_n')  + \epsilon \nonumber\\
    \trm{if}\qquad n &\geq \f{c}{\epsilon^2} \log \f{N(\epsilon, \Gcal^2)}{\delta}
\end{align}
where $c$ is a constant. The quantity $N(\epsilon, \Hcal)$ is the $\epsilon$-net covering number of the hypothesis class $\Hcal$ under a given metric $m$ \citesi{baxter2000model}. According to Baxter's result, for our case with two tasks, $p$ and $q^k$, we are interested in computing the covering number for $\Hcal = \Gcal \times \Gcal$ and the metric $m$ between two functions in $g, g' \in \Gcal^2$ as
\[
    m(g, g') = \f{1}{2} \int \abr{d(g(x_1), f(x_1)) + d(g(x_2), f(x_2)) - d(g'(x_1), f(x_1)) - d(g'(x_2), f(x_2))}\ \trm{d}p(x_1)\ \trm{d} q^k(x_2).
\]
with the labels $f(x_i)$ given by the teacher. Our hypothesis class $\Hcal$ on the two tasks is the Cartesian product of the hypothesis class $\Gcal$.
Haussler's theorem~\citesi{haussler1990probably} gives an upper bound on the covering number in terms of the VC-dimension
\begin{align}
    \label{eq:haussler}
    \log N(\epsilon, \Hcal) \leq 2 V_\Hcal \log(c/\epsilon).
\end{align}
where $V_\Hcal$ is the VC-dimension~\citesi{vapnik2015uniform} of $\Hcal$, $c > 1$ is a constant, and $\log N(\epsilon, \Hcal)$ is also called the metric entropy.

Observe that the left-hand side in~\cref{eq:baxter} can be written as
\begin{align}
    \label{eq:lhs}
    D\rbr{f, g, \f{p+q^k}{2}} &= D(f, g, p) + \f{1}{2} \int d(f(x), g(x))\  \trm{d}(q^k-p).
\end{align}
Let us define
\[
    \Delta_k := \norm{T^k_q p - p}_\trm{TV}.
\]
We next analyze the metric $m(g,g')$ where we note that again $q^k = T^k_q p$.
\[
    \begin{aligned}
        m(g,g') &\leq \int \abr{d(g(x), f(x)) - d(g'(x), f(x))}\ \trm{d}p(x)\\
        &+ \f{1}{2} \int \abr{d(g(x), f(x)) - d(g'(x), f(x))}\ \trm{d}\rbr{T^k_q - p}(x)\\
        &\leq \int \abr{d(g(x), f(x)) - d(g'(x), f(x))}\ \trm{d}p(x) + \norm{T^k_q p - p}_\trm{TV}.
     \end{aligned}
\]
Similarly we also have
\[
    \int \abr{d(g(x), f(x)) - d(g'(x), f(x))}\ \trm{d}p(x) - \norm{T^k_q p - p}_\trm{TV} \leq m(g,g').
\]
In other words, the distance $m(g,g')$ on the Cartesian space $\Gcal^2$ can be upper bounded by the distance between $g,g'$ on the original space $\Gcal$ up to an additive term $\Delta_k$ that increases with the number of steps $k$ of Gibbs sampling.

Next observe that we have an upper bound on the metric entropy
\begin{equation}
    \label{eq:metric_ent_cartesian}
    \log N(\epsilon, \Gcal^2) \leq 2 \log N(\epsilon, \Gcal)
\end{equation}
if the two datasets $X_n$ and $X_n'$ are iid. If the datasets are not iid, using the calculation for $m(g,g')$ above, computing the size of the $\epsilon$-net for $\Gcal^2$ is effectively the same as changing $\epsilon$ on the right hand side of~\cref{eq:metric_ent_cartesian} to
\[
    \begin{aligned}
        \epsilon' &= \epsilon - \Delta_k\\
        &= \epsilon \rbr{1 - \f{\Delta_k}{\epsilon}}.
    \end{aligned}
\]
Plugging the previous two expressions into~\cref{eq:haussler} implies
\begin{align}
    \log N(\epsilon, \Gcal^2) &\leq 2 \log N(\epsilon, \Gcal)
    \nonumber\\
    &\approx 4 V_\Gcal \log \rbr{\f{c}{\epsilon} \rbr{1 + \f{\Delta_k}{\epsilon}}} \nonumber\\
    &= 4 V_\Gcal \rbr{\log(c/\epsilon) + \f{\Delta_k}{\epsilon}} + o((\Delta_k/\epsilon)^2).
\end{align}
The approximation above is valid if we additionally assume $\Delta_k \ll \epsilon$. We have thus shown that there exists a constant $V$ such that with probability at least $1-\delta$:
\begin{align}
    D(f, g^*, p) &\leq D_{\trm{emp}}(f, g^*, X_n \cup X_n') + \sqrt{\f{4 V(c + \Delta_k) - \log \delta}{n}} + \f{1}{2} \int d(f(x), g^*(x))\  \trm{d}(p-q^k) \nonumber\\
    &\leq D_{\trm{emp}}(f, g^*, X_n \cup X_n') + \sqrt{\f{4 V(c + \Delta_k) - \log \delta}{n}} + \Delta_k
    \label{eq:}
\end{align}
The inequality follows because
\[
    \abr{\int d(f(x), g^*(x))\  \trm{d}(p-q^k)} \leq \abr{\int \trm{d}(p-q^k)} \leq 2 \norm{p-q^k}_{\trm{TV}}
\]
since $d(\cdot, \cdot) \leq 1$. Recall $k$ is the number of Gibbs sampling steps, $c$ is a constant, and  $\displaystyle g^* = \argmin_g\  D_{\trm{emp}}(f, g, X_n \cup X_n')$.
\end{proof}

We provide some additional comments on this result. Note that $\Delta_k \to \norm{p-q}_{\trm{TV}}$ as $k \to \infty$, so it increases with the number of Gibbs sampling steps $k$. We can draw a large number of samples $n$ from $q^k$ to reduce the second term in the bound. 
Using a large $k$ is both computationally inefficient and 
% Larger the Gibbs steps $k$ from which these samples are drawn, lower is the efficiency; this is seen because $\Delta_k$ increases with $k$. These samples 
may also cause a bias given by the additive term of $\Delta_k$ (third term), if the stationary distribution $q$ of our Gibbs sampler poorly approximates $p$. 
As our pseudolikelihood model is fit to limited data in practice, it is thus better to draw a large number of samples from earlier steps, i.e.\  using only a few steps of Gibbs sampling from each training datum instead of running a long chain.
Among all $k$ that produce samples which are approximately independent of the original training data, we would like to use the smallest. 

The experiments in our paper empirically show that, on an average over many datasets, running the Gibbs sampler for 1--5 rounds (one round involves performing a Gibbs step for every conditional in the pseudolikelihood) works better than running it for longer. Note that if we employ fewer steps than even a single round of Gibbs sampling, the augmented data will be highly dependent on the training data as some features will not have been resampled, thus diminishing the \emph{effective sample size} of the student's distillation dataset. 
It is also readily seen from the above bound that if the Gibbs sampler is initialized at a distribution other than $p$, we would need a large number of steps $k$ before the bias term $\norm{T^k_q \nu - p}_{\trm{TV}}$ is adequately small.

\clearpage
\bibliographystylesi{abbrvnat}
\bibliographysi{distill}

\end{document}